%% file: main.tex
\documentclass{article} 
\usepackage{iclr2026_conference,times}

\usepackage{hyperref}
\usepackage{url}
\usepackage{multirow,booktabs} 
\usepackage[table,xcdraw, dvipsnames]{xcolor}
\usepackage{multicol}
\usepackage{graphicx}
\usepackage{amsthm}
\usepackage{mathtools}
\usepackage{amssymb}
\usepackage{wrapfig}
\usepackage{enumitem}
\usepackage{ulem}

\theoremstyle{plain}
\newtheorem{thm}{Theorem}[section]

\newtheorem{lem}[thm]{Lemma}
\newtheorem{cor}[thm]{Corollary}
\newtheorem{assumption}[thm]{Assumption}
\newtheorem{defn}[thm]{Definition}

\definecolor{tabblue}{HTML}{1f77b4}
\definecolor{taborange}{HTML}{ff7f0e}
\definecolor{tabgreen}{HTML}{2ca02c}
\definecolor{tabred}{HTML}{d62728}
\definecolor{tabpurple}{HTML}{9467bd}
\definecolor{tabbrown}{HTML}{8c564b}
\definecolor{tabpink}{HTML}{e377c2}
\definecolor{tabgray}{HTML}{7f7f7f}
\definecolor{tabolive}{HTML}{bcbd22}
\definecolor{tabcyan}{HTML}{17becf}

\hypersetup{
    colorlinks=true,       
    linkcolor=tabred,      
    citecolor=tabblue,     
    filecolor=magenta,     
    urlcolor=tabgreen         
}

\input{acronyms}

\input{math_commands.tex}

\title{{ResCP}: Reservoir Conformal Prediction for Time Series Forecasting}

\author{Roberto Neglia\thanks{Equal contribution. Correspondance to: \texttt{roberto.neglia@uit.no}; \texttt{andrea.cini@usi.ch}.}\\
UiT the Arctic University of Norway \\
\And
Andrea Cini\footnotemark[1] \\
IDSIA USI-SUPSI, Università della Svizzera italiana \\
Swiss National Science Foundation Postdoctoral Fellow \\
\And
Michael M.~Bronstein \\
University of Oxford \\
AITHYRA \\
\And
Filippo Maria Bianchi \\
UiT the Arctic University of Norway \\
NORCE Norwegian Research Centre AS \\
}

\iclrfinalcopy 
\begin{document}

\maketitle

\input{abstract}
\input{content}
\bibliography{biblio}
\bibliographystyle{biblio}
\newpage

\appendix
\section*{\Large Appendix}
\input{appendix}
\input{appendix_datasets}
\input{appendix_hyperparameters}
\input{appendix_additional_alpha}
\input{appendix_ablation}
\input{appendix_sensitivity}
\input{appendix_hw_sw}
\input{llms}

\end{document}

%% file: acronyms.tex
\usepackage[acronym, nohypertypes={acronym}]{glossaries}

\newacronym{method}{\textsc{CoRel}}{\textit{\underline{Co}nformal \underline{Rel}ational Prediction}\nobreak}
\newacronym{gdl}{GDL}{graph deep learning}
\newacronym{cp}{CP}{conformal prediction}
\newacronym{lcp}{LCP}{localized CP}
\newacronym{scp}{SCP}{{split CP}}
\newacronym{fcp}{FCP}{{full CP}}
\newacronym{iqn}{IQN}{implicit quantile network}
\newacronym{relqn}{RelQP}{\textit{relational quantile predictor}}
\newacronym{pi}{PI}{prediction interval}

\newacronym{rescp}{\textsc{ResCP}}{\textit{\underline{Res}ervoir \underline{C}onformal \underline{P}rediction}}
\newacronym{rescqr}{\textsc{ResCQR}}{Reservoir Conformal Quantile Regression}

\newglossaryentry{scpi}{name=SPCI,description=}
\newglossaryentry{nexcp}{name=NexCP,description=}
\newglossaryentry{kowcpi}{name=KOWCPI,description=}
\newglossaryentry{seqcp}{name=SeqCP,description=}
\newglossaryentry{hopcpt}{name=HopCPT,description=}
\newglossaryentry{cornn}{name=\textsc{CP-QRNN},description=}
\newglossaryentry{adam}{name=Adam,description=}

\newacronym{stgnn}{STGNN}{spatiotemporal graph neural network\nobreak}
\newacronym{gnn}{GNN}{graph neural network}
\newacronym{mp}{MP}{message-passing}
\newacronym{mlp}{MLP}{multilayer perceptron}
\newacronym{tts}{TTS}{time-then-space}
\newacronym{stt}{STT}{space-then-time}
\newacronym{tas}{T\&S}{time-and-space}
\newacronym{rnn}{RNN}{recurrent neural network}
\newacronym{tcn}{TCN}{temporal convolutional network}
\newacronym{arima}{ARIMA}{Autoregressive Integrated Moving Average}
\newacronym{gru}{GRU}{gated recurrent unit}
\newacronym{rc}{RC}{reservoir computing}
\newacronym{esn}{ESN}{echo state network}
\newacronym{esp}{ESP}{Echo State Property}
\newacronym{sn}{SN}{\textit{sensor network}}
\newacronym{iot}{IoT}{Internet of Things}
\newacronym{stmp}{STMP}{spatiotemporal message-passing}
\newacronym{stcn}{STCN}{spatiotemporal convolutional network}
\newacronym{gcrnn}{GCRNN}{graph convolutional recurrent neural network}
\newacronym{mae}{MAE}{mean absolute error}
\newacronym{mape}{MAPE}{mean absolute percentage error}
\newacronym{rmse}{RMSE}{root mean squared error}
\newacronym{mre}{MRE}{mean relative error}
\newacronym{cdf}{CDF}{cumulative distribution function}
\newacronym{ar}{AR}{autoregressive}

\glsdisablehyper
\newglossaryentry{ttsimp}{name=RNN\texttt{+}IMP,description=}
\newglossaryentry{ttsamp}{name=RNN\texttt{+}AMP,description=}
\newglossaryentry{tasimp}{name=GCRNN-IMP,description=}
\newglossaryentry{tasamp}{name=GCRNN-AMP,description=}

\newglossaryentry{fcrnn}{name=FC-RNN,description=}
\newglossaryentry{local_rnn}{name=LocalRNNs,description=}
\newglossaryentry{dcrnn}{name=DCRNN,description=}
\newglossaryentry{agcrn}{name=AGCRN,description=}
\newglossaryentry{gwnet}{name=GraphWaveNet,description=}

\newglossaryentry{gpvar}{name=GPVAR,description=}
\newglossaryentry{lgpvar}{name=GPVAR-L,description=}
\newglossaryentry{la}{name=METR-LA,description=}
\newglossaryentry{bay}{name=PEMS-BAY,description=}
\newglossaryentry{cer}{name=CER-E,description=}
\newglossaryentry{air}{name=AQI,description=}
\newglossaryentry{pems3}{name=PEMS03,description=}
\newglossaryentry{pems4}{name=PEMS04,description=}
\newglossaryentry{pems7}{name=PEMS07,description=}
\newglossaryentry{pems8}{name=PEMS08,description=}

%% file: math_commands.tex
\usepackage{amsmath,amsfonts,bm, dsfont}

\usepackage{etoolbox}
\makeatletter
\patchcmd{\hyper@makecurrent}{%
    \ifx\Hy@param\Hy@chapterstring
        \let\Hy@param\Hy@chapapp
    \fi
}{%
    \iftoggle{inappendix}{
        \@checkappendixparam{chapter}%
        \@checkappendixparam{section}%
        \@checkappendixparam{subsection}%
        \@checkappendixparam{subsubsection}%
        \@checkappendixparam{paragraph}%
        \@checkappendixparam{subparagraph}%
    }{}%
}{}{\errmessage{failed to patch}}

\newcommand*{\@checkappendixparam}[1]{%
    \def\@checkappendixparamtmp{#1}%
    \ifx\Hy@param\@checkappendixparamtmp
        \let\Hy@param\Hy@appendixstring
    \fi
}
\makeatletter

\newtoggle{inappendix}
\togglefalse{inappendix}

\apptocmd{\appendix}{\toggletrue{inappendix}}{}{\errmessage{failed to patch}}

\def\eqref#1{equation~\ref{#1}}

\def\1{\bm{1}}

\def\vh{{\bm{h}}}

\def\vw{{\bm{w}}}
\def\vx{{\bm{x}}}

\def\mW{{\bm{W}}}

\DeclareMathAlphabet{\mathsfit}{\encodingdefault}{\sfdefault}{m}{sl}
\SetMathAlphabet{\mathsfit}{bold}{\encodingdefault}{\sfdefault}{bx}{n}

\def\sN{{\mathbb{N}}}

\def\sR{{\mathbb{R}}}

\newcommand{\Var}{\mathrm{Var}}

\newcommand{\Cov}{\mathrm{Cov}}

\DeclareMathOperator*{\argmin}{arg\,min}

%% file: abstract.tex
\begin{abstract} 
    \Gls{cp} offers a powerful framework to build distribution-free prediction intervals for exchangeable data. 
    Existing methods that extend \gls{cp} to sequential data rely on fitting a relatively complex model to capture temporal dependencies. 
    However, these methods can fail if the sample size is small and often require expensive retraining when the underlying data distribution changes. 
    To overcome these limitations, we propose \acrfull{rescp}, a novel training-free \gls{cp} for time series. 
    Our approach leverages the efficiency and representation capabilities of Reservoir Computing to dynamically reweight conformity scores. 
    In particular, we compute similarity scores across reservoir states and use them to adaptively reweight the observed residuals. \Acrshort{rescp} enables us to account for local temporal dynamics when modeling the error distribution without compromising computational scalability. We prove that, under reasonable assumptions, \acrshort{rescp} achieves asymptotic conditional coverage, and we empirically demonstrate its effectiveness across diverse forecasting tasks. 
\end{abstract}

%% file: content.tex
\section{Introduction}\label{sec:introduction} 

Despite deep learning often achieving state-of-the-art results in time series forecasting, widely used methods do not provide a way to quantify the uncertainty of the predictions~\citep{benidis2022deep}, which is crucial for their adoption in risk-sensitive scenarios, such as healthcare~\citep{makridakis2019forecasting}, load forecasting~\citep{hong2016probabilistic}, and weather forecasting~\citep {palmer2006predictability}.
Moreover, in many cases, \glspl{pi} must not only be reliable but also fast to compute.
Existing uncertainty quantification approaches for time series often rely on strong distributional assumptions~\citep{benidis2022deep, salinas2020deepar}, which do not always fit real-world data. 
Moreover, they often require long and expensive training procedures, as well as modifications to the underlying forecasting model, which limits their applicability to large datasets.
A framework that recently gained attention in uncertainty quantification is \gls{cp}~\citep{vovk2005algorithmic, angelopoulos2023conformal}. 
\gls{cp} assumes exchangeability between the data used to build the \glspl{pi} and the test data points, meaning that the joint distribution of the associated sequence of random variables does not change when indices are permuted.
Clearly, this does not hold for time series, as the presence of temporal dependencies and, possibly, non-stationarity violates the assumption. 
Moreover, temporal dynamics can be heterogeneous and often result in heteroskedastic errors. 
This requires mechanisms to make \glspl{pi} locally adaptive~\citep{lei2018distribution, guan2022localized}.

To extend \gls{cp} to time series data, a popular approach is to reweight the observed residuals to handle non-exchangeability and/or temporal dependencies~\citep{barber2023conformal, tibshirani2019conformal, auer2023conformal}. 
In particular, \cite{auer2023conformal} proposes an effective reweighting mechanism based on a learned model and soft-attention operators, which, however, suffers from the high computational cost typical of Transformer-like architectures. 
Other methods build \glspl{pi} by directly learning the quantile function of the error distribution~\citep{jensen2022ensemble, chen2023sequential, lee2024transformer, lee2025kernelbased, cini2025relational}. This, however, often results in high sample complexity and might require frequent model updates to adapt to changes in the data distribution. 

To address these limitations, we propose \gls{rescp}, a novel \gls{cp} method that leverages a reservoir, i.e., a randomized recurrent neural network~\citep{jaeger2001esn, lukosevicius2009reservoir}, driven by the observed prediction residuals to efficiently compute data-dependent weights adaptively at each time step. 
To capture locally similar dynamics, the weights are derived from the similarity between current and past states of the reservoir, which is implemented as an \gls{esn}~\citep{jaeger2001esn, gallicchio2020deep}. 
This enables \gls{rescp} to model local errors through a weighted empirical distribution of past residuals, where residuals associated with similar temporal dynamics receive higher weights. 
The \gls{esn} does not require any training and is initialized to yield stable dynamics and informative representations of the input time series at each time step. 
This makes our approach extremely scalable and easy to implement.
\gls{rescp} can be applied on top of \textit{any} point forecasting model, since it only requires residuals from a disjoint calibration set.

Our main contributions are summarized as follows.
\begin{itemize}[topsep=.1em, leftmargin=1em, itemindent=1em]
    \item We provide the first assessment of \acrlong{rc} for \gls{cp} in time series analysis and show that it is a valid and scalable alternative to existing methods. 
    \item We introduce \gls{rescp}, a novel, scalable, and theoretically grounded tool for distribution-free uncertainty quantification in sequential prediction tasks.
    \item We prove that \gls{rescp} can achieve asymptotic conditional coverage under reasonable assumptions on the data-generating process. 
    \item We introduce variants of \gls{rescp} to handle distribution shifts and leverage exogenous variables.
\end{itemize}
We evaluate \gls{rescp} on data from real-world applications and show its robustness and effectiveness by comparing it to state-of-the-art baselines.

\section{Background and related work}\label{sec:background} 

We consider a forecasting setting where the goal is to predict future values of a time series based on past observations. 
Let $\{y_t\}_{t=1}^T$ be the sequence of scalar target observations at each time step $t$. 
We denote by $y_{1:T}$ the entire history of observations up to time $T$. 
Moreover, let $\{\bm{u}_t\}_{t=1}^T$ be the sequence of exogenous variables at each time step $t$ of dimension $D_u$.
In this context, we are also given a forecasting model $\hat{y}_{t+H} = \hat{f}(y_{t-W:t}, \bm{u}_{t-W:t})$ that produces point forecasts of $H$-steps-ahead, with $H \geq 1$, given a window $W \geq 1$ of past observations. The model $\hat{f}$ can be any kind of forecasting model, e.g., a \gls{rnn}. The objective is to construct valid \glspl{pi} that reflect the uncertainty in the forecasts. 
Ideally, \glspl{pi} should be conditioned on the current state of the system, i.e., we want to achieve conditional coverage:
\begin{equation}
    \mathbb{P}\left( y_{t+H} \in \hat{C}^\alpha_{T}(\hat{y}_{t+H}) \big\vert y_{\leq t}, \bm{u}_{\leq t}\right) \geq 1 - \alpha
\end{equation}
where $\hat{C}^\alpha_{T}(\hat{y}_{t+H})$ is the prediction interval for the forecast $\hat{y}_{t+H}$ at significance level $\alpha$.

\subsection{Conformal Prediction in Time Series Forecasting}

\Gls{cp}~\citep{vovk2005algorithmic, angelopoulos2023conformal} is a framework for building valid \textit{distribution-free} \glspl{pi}, possibly with finite samples. 
\Gls{cp} provides a way to quantify uncertainty in predictions by leveraging the empirical quantiles of \textit{conformity} scores, which in our regression setting we define as the prediction residuals
\begin{equation}
    r_t = y_t - \hat{y}_t.
\end{equation}
These scores measure the discrepancy between the model predictions and the observed values.
The \gls{cp} framework we consider in this work is \gls{scp} \citep{vovk2005algorithmic}, which constructs valid \glspl{pi} post-hoc by leveraging a disjoint calibration set to estimate the distribution of prediction errors. 
In its standard formulation, \gls{cp} treats all calibration points symmetrically when building \glspl{pi}. 
In heterogeneous settings, this can lead to overly conservative intervals as conformal scores are often heteroskedastic. 
To address heteroskedasticity, \gls{cp} can be made \textit{locally adaptive} by providing \glspl{pi} whose width can adaptively shrink and inflate at different regions of the input space~\citep{lei2018distribution}. 
In this context, \cite{guan2022localized} proposed \gls{lcp}, which reweights conformity scores according to the similarity between samples in feature space, given a localization function. \cite{hore2023conformal} introduced a randomized version of \gls{lcp}, with approximate conditional coverage guarantees and more robust under covariate shifts. 
Although this reweighting breaks exchangeability, validity is recovered through a finite-sample correction of the target coverage level. 
Other approaches rely on calibrating estimates of the spread of the data~\citep{lei2013distributionfree} or quantile regression~\citep{romano2019conformalized, jensen2022ensemble, feldman2023calibrated}, but often require ad-hoc base predictors~(i.e., they cannot be applied on top of generic point predictors).

\paragraph{\Acrlong{cp} for time series} 
Under the assumption of data exchangeability, \gls{cp} methods provide, in finite samples, valid intervals given a specified level of confidence \citep{angelopoulos2023conformal}. 
However, in settings where exchangeability does not hold, standard \gls{cp} methods can fail in providing valid coverage~\citep{barber2023conformal, tibshirani2019conformal}. 
This is the case of time series data where temporal dependencies break the exchangeability assumption. 
\cite{barber2023conformal} show that reweighting the available residuals~(in a data-independent fashion) can handle non-exchangeability and distribution shifts, and introduce \gls{nexcp}, a weighting scheme that exponentially decays the past residuals over time. 
\cite{barber2025unifying} provides a unified view of several weighted \gls{cp} methods as approaches to condition uncertainty estimates on the available information on the target data point. 
Other recent approaches, instead, fit a model to account for temporal dependencies and rely on asymptotic guarantees under some assumptions on the data-generating process~\citep{chen2023sequential, auer2023conformal, lee2025kernelbased}. 
For example, \gls{scpi}~\citep{chen2023sequential} uses a quantile random forest trained at each step on the most recent scores. 
Conversely, \gls{hopcpt} leverages a Hopfield network to learn a data-dependent reweighting scheme based on soft attention~\citep{auer2023conformal}. 
\cite{lee2025kernelbased} use a reweighted Nadaraya-Watson estimator to perform quantile regression over the past nonconformity scores and derive the weights from the kernel function. 
\cite{cini2025relational} proposes instead to use graph neural networks~\citep{bacciu2020gentle, bronstein2021geometric} to additionally condition uncertainty estimates on correlated time series. As already mentioned, we focus on methods that estimate uncertainty on the prediction of a pre-trained model; we refer to \cite{benidis2022deep} for a discussion of probabilistic forecasting architectures. 

\subsection{Echo State Networks}

\Glspl{esn} are among the most popular \gls{rc} methods~\citep{jaeger2001esn, lukosevicius2009reservoir}. In particular, \glspl{esn} are \glspl{rnn} that are randomly initialized to ensure the stability and meaningfulness of the representations and left untrained.
An \gls{esn} encodes input sequences into nonlinear, high-dimensional state representations through the recurrent component, called \textit{reservoir}, which extrapolates from the input a rich pool of dynamical features.
The state update equation of the reservoir is
\begin{equation}\label{eq:esn_update}
    \bm{h}_t = (1-l)\bm{h}_{t-1} + l \, \sigma\left(\mW_x \bm{x}_t + \mW_h \bm{h}_{t-1} + \bm{b} \right),
\end{equation}
where $\bm{x}_t$ denotes the input at time $t$, $\mW_x \in \mathbb{R}^{D_h \times D_x}$ and $\mW_h \in \mathbb{R}^{D_h \times D_h}$ are fixed random weight matrices, $\bm{b} \in \mathbb{R}^{D_h}$ is a random bias vector, $\bm{h}_t \in \mathbb{R}^{D_h}$ is the reservoir state, $l\in(0,1]$ is the leak rate which controls how much of the current state is retained at each update, and $\sigma$ is a nonlinear activation, typically the hyperbolic tangent. 
When $\mW_x$ and $\mW_h$ are properly initialized~(see,e.g., \citealt{lukovsevivcius2012practical}), reservoir states $\bm{h}_t$ provide expressive embeddings of the past inputs $\bm{x}_{\leq t}$, which can be processed by a trainable readout to perform downstream tasks such as time series classification or  forecasting~\citep{bianchi2020reservoir}.
Since embeddings already model nonlinear dynamics, the readout is usually implemented as a simple linear layer which can be trained efficiently. 

\paragraph{Echo state property} 
The reservoir dynamics should be neither too contractive nor chaotic.
In the first case, the reservoir produces representations that are not sufficiently expressive. 
In the latter, the reservoir is unstable and responds inconsistently to nearly identical input sequences.
\Glspl{esn} are usually configured to ensure the \gls{esp}, a necessary condition for global asymptotic stability under which the reservoir state asymptotically forgets its initial conditions and the past inputs. 
Formally, for any initial states $\bm{h}_0, \bm{{h}}^\prime_0$ and input sequence $\bm{x}_{1:T}$, by calling $f_R(\bm{h}_0, \bm{x}_{1:T})$ the final state of the reservoir initialized with $\bm{h}_0$ and being fed with $\bm{x}_{1:T}$, then the \gls{esp} is defined as
\begin{equation}
    || f_R(\bm{h}_0, \bm{x}_{1:T}) - f_R(\bm{{h}^\prime}_0, \bm{x}_{1:T}) || \to 0 \enskip \text{as} \enskip T \to \infty
\end{equation}
The \gls{esp} can be achieved by properly initializing $\mW_h$ and by setting its spectral radius ${\rho(\mW_h) < 1}$~\citep{yildiz2012re, gallicchio2020deep}.

\section{Reservoir Conformal Prediction}\label{sec:methodology}
We propose a novel approach for computing \gls{cp} intervals for time series forecasting based on the representations generated by \gls{rc} to model temporal dependencies. 
Our method leverages \gls{cp} to provide accurate uncertainty estimates while maintaining the computational efficiency of \gls{rc}. 

\input{imgs/method_img}

\subsection{Reservoir Sampling Conformal Prediction}\label{sec:rescp}
We construct \glspl{pi} from a calibration set consisting of a time series of residuals of length $T$. 
In particular, we aim at reweighting the residuals based on similarities with the local dynamics at the target time step. 
We leverage the reservoir capability of embedding the observed dynamics $\bm{x}_{\leq t}$ into its state $\bm{h}_t$. An overview of the whole procedure is depicted in Figure~\ref{fig:overview}. 
In general, the input $\bm{x}_t$ can include any set of endogenous and exogenous variables available at time step $t$. 
In the following, unless differently stated, we consider $\bm{x}_t$ to be the prediction residual~($\vx_t=r_t$ obtained by training a generic forecasting model on a disjoint training set); the discussion on how to effectively incorporate exogenous variables is deferred to Section~\ref{sec:rescqr}.

Given an \gls{esn} with dynamics defined in Equation~\ref{eq:esn_update} and hyperparameters $\bm\theta$, the time series ${\vx_{1:t} = (\vx_1, \dots, \vx_t)}$ is embedded into a sequence of $t$ reservoir states which encode observed dynamics at each time step as
\begin{align}
    \vh_{t} = \text{ESN}_{\bm\theta}(\vx_{1:t}), \qquad \vh_{t} \in{\mathbb{R}^{D_h}}.
\end{align}
In particular, to compute the \gls{pi} for the $H$-step-ahead prediction $\hat{y}_{t+H}$, we take as \textit{query state} the current state $\bm{h}_t$ (since we do not have access to states $\bm{h}_{>t}$) and compute similarity scores between $\bm{h}_t$ and the states in the calibration set. 
Note that in most scenarios, the calibration set can be iteratively updated as new observations become available~(i.e., we can include in the calibration set all the observations up to the current time step $t$).
Weights are computed as
\begin{align}
    z_{s}(\vh_t) & = \textsc{Sim}(\bm{h}_t, \vh_{s}), \qquad 1\leq s \leq T-H \\
    \left\{ w_{1}(\vh_t),\dots, w_{T-H}(\vh_t) \right\} & = \textsc{SoftMax}\left\{ \frac{z_{1}(\vh_t)}{\tau},\dots, \frac{z_{T-H}(\vh_t)}{\tau} \right\}\label{eq:softmax}
\end{align}
where \textsc{Sim} is a similarity score (e.g., dot product) and $\tau>0$ is a temperature hyperparameter. 
Our goal is to approximate the conditional \gls{cdf} of the residual at time $t+H$ with
\begin{equation}\label{eq:rescp-cdf}
    \widehat{F}(r \mid \vh_t) \coloneqq \sum_{s=1}^{T-H} w_s(\vh_t) \mathds{1}\big(r_{s+H}\leq r\big) \approx \mathbb P(r_{t+H} \le r \mid \vx_{\leq t})
\end{equation}
where $w_s(\vh_t)$ is a weight proportional to the similarity of $\vh_s$ to $\vh_t$  and such that $\sum_s w_s(\vh_t)=1$.
Quantiles of the conditional error distribution can then be estimated as:
\begin{equation}\label{eq:empirical_quantile}
    \widehat{Q}_\beta(\vh_t) \coloneqq \inf\{r \in \mathbb{R} \;:\; \widehat{F}(r \;\vert\;\vh_t) \geq \beta\}.
\end{equation}
In practice, instead of computing the weighted quantile in \autoref{eq:empirical_quantile} directly, we approximate it through Monte Carlo sampling akin to \cite{auer2023conformal}. 
More specifically, we sample residuals accordingly to the weights in \autoref{eq:softmax} and compute the standard empirical $\beta$-quantile on the sampled residuals:
\begin{equation}\label{eq:rescp_quantiles}
    \hat{q}^{\alpha/2}_{t+H} = \widehat{Q}_{\alpha/2}(\vh_t), \qquad \hat{q}^{1-\alpha/2}_{t+H} = \widehat{Q}_{1-\alpha/2}(\vh_t).
\end{equation}
\Glspl{pi} with the desired confidence level $\alpha$ are obtained as:
\begin{equation}\label{eq:rescw_int}
    \widehat{C}_{T}^\alpha(\hat{y}_{t+H}) = \left[\hat{y}_{t+H} + \hat{q}^{\alpha/2}_{t+H}, \hat{y}_{t+H} + \hat{q}^{1-\alpha/2}_{t+H}\right].
\end{equation}
To account for skewed error distributions, we follow the approach of \cite{xu2023conformal} and refine the interval in \ref{eq:rescw_int} by selecting the level $\beta^*$ that minimizes its width:
\begin{align}
    \beta^* &= \argmin_{\beta \in [0,\alpha]} \left[\widehat{Q}_{1-\alpha+\beta}(\vh_t) - \widehat{Q}_{\beta}(\vh_t)\right]\notag \\ 
    \widehat{C}_{T}^{\alpha}(\hat{y}_{T+H}) &= \left[\hat{y}_{T+H} + \widehat{Q}_{\beta^*}(\vh_t), \hat{y}_{T+H} + \widehat{Q}_{1-\alpha+\beta^*}(\vh_t)\right].\label{eq:rescw_int_beta}
\end{align}
This can provide narrower \glspl{pi}, at the cost of searching for the optimal level $\beta$~\citep{xu2023conformal}.

\paragraph{Discussion} Essentially, \gls{rescp} is a local \gls{cp} method for time series where local similarity is gauged by relying on reservoir states.  
By sampling more residuals associated with similar states, we condition the estimates on specific dynamics of interest in the time series.
This allows us to build locally adaptive intervals. 
While \cite{guan2022localized} provides asymptotic guarantees for standard local \gls{lcp} for i.i.d.\ data, such results cannot be applied in the context of time series. 
In the following section, we discuss the conditions that allow \gls{rescp} to provide~(asymptotically) valid intervals. Notably, this will require assumptions on the nature of the process generating the data and on the reservoir dynamics. Finally, note that with finite samples, even with i.i.d.\ observation, data-dependent weights might introduce a bias that should be accounted for~\citep{guan2022localized}.
In practice, we found that tuning \gls{rescp} hyperparameters on a validation set was sufficient to provide accurate coverage in most scenarios (see \autoref{sec:experiments}).

\subsubsection{Theoretical analysis} \label{sec:theory}
We analyze the theoretical properties of \gls{rescp}, and show in Corollary~\ref{cor:asym_cond_cov} that it asymptotically guarantees coverage under some regularity assumptions. 
As already mentioned, guarantees of standard \gls{cp} methods break down for time series unless one relies on highly unrealistic assumptions~(e.g., exchangeable or i.i.d.\ observations)~\citep{barber2023conformal}. 
Moreover, even assuming exchangeability, it is impossible to construct finite-length \glspl{pi} with distribution-free conditional coverage guarantees~\citep{lei2013distributionfree, vovk2012conditional}.
We start by introducing the assumptions needed to prove the consistency of the weighted empirical \gls{cdf} in \autoref{eq:rescp-cdf}. 

\begin{assumption}[Time-invariant and mixing process]
\label{ass:mixing}
We assume that the process ${\{Z_t = (\vx_t, r_{t+H})\}_{t=1}^\infty}$ is time-invariant and strongly mixing ($\alpha$-mixing) with coefficient $\alpha(k) \to 0$ as $t\to \infty$. 
\end{assumption}
Intuitively, Assumption~\ref{ass:mixing} implies that the system’s dynamics do not change over time and that it ``forgets’’ its initial conditions and past structure, thereby allowing observations that are far apart in time to be treated as if they were independent.
\begin{assumption}[\gls{esp} and Lipschitz properties in \gls{esn}]
\label{ass:esn}
Let $\vx_{\leq t}=(\ldots,\vx_{t-1},\vx_{t})$ denote the history. 
Define the \emph{fading-memory metric} for $\gamma\in(0,1)$ by
\begin{equation*}
    d_{\mathrm{fm}}(\vx_{\leq t},\vx_{\leq t}^\prime)
    \coloneqq \sum_{k=1}^\infty \gamma^k \, \|\vx_{t-k+1}-\vx^\prime_{t-k+1}\|_2.
\end{equation*}
Then the \gls{esn} state map ${\text{ESN}_{\bm\theta}: \vx_{\leq t}\mapsto {\vh}_t\in\mathbb R^{D_h}}$ is well-defined and causal with fading memory (i.e., the output only depends on the past and present inputs), and there exist a constant $L_X > 0$ such that for all histories $\vx_{\leq t},\vx_{\leq t}'$:
\begin{equation*}
    \| \text{ESN}_{\bm\theta} (\vx_{\leq t})-\text{ESN}_{\bm\theta}(\vx_{\leq t}^\prime) \|_2
    \;\le\; L_X \, d_{\mathrm{fm}}(\vx_{\leq t},\vx_{\leq t}^\prime).
\end{equation*}

\end{assumption}
Assumption~\ref{ass:esn} tells us that reservoir has stable and contractive dynamics: past inputs influence the state in a controlled, exponentially decaying way, and small perturbations are eventually washed out from the reservoir's state. These assumptions are commonly used in \gls{rc}~\citep{grigoryeva2019differentiable, gallicchio2011architectural}. 

\begin{assumption}[Continuity of the conditional law]
\label{ass:continuity}
Let $\mathbb P(r_{t+H}\in \sR \mid \vx_{\leq t})$ denote the conditional law of
$r_{t+H}$ given the past. Then for every $r\in\mathbb R$ the map
\begin{equation*}
    \vx_{\leq t} \;\mapsto\; F(r\mid \vx_{\leq t})
\end{equation*}
is continuous in $\vx_{\leq t}$ with respect to $d_{\mathrm{fm}}$, i.e.
\begin{equation*}
    d_{\mathrm{fm}}(\vx_{\leq t},\vx_{\leq t}^\prime) \to 0 \implies \sup_r \vert F(r\mid \vx_{\leq t}) - F(r\mid \vx_{\leq t}^\prime) \vert \to 0
\end{equation*}
\end{assumption}

Intuitively, Assumption~\ref{ass:continuity} ensures that if two histories are similar~(in the fading-memory sense), then the distributions of the future residuals conditioned on these histories are also close. This is a reasonable assumption to enable learning. 

\begin{defn}[Effective sample size]
    Let $n$ be the number of available calibration samples, the \textit{effective sample size}~\citep{kong1994sequential, liu2008monte} is defined in our setting as
    \begin{equation}
        m_n \coloneqq \left(\sum_{i=1}^n (w^{(n)}_i)^2\right)^{-1}.
    \end{equation}
\end{defn}
Note that at high temperature, weights are approximately uniform~($w^{(n)}_i \approx \frac{1}{n}$) and $\sum_{i} (w^{(n)}_i)^2 \approx \frac{1}{n}$, hence $m_n \approx n$, 
which means that the method collapses to vanilla \gls{scp} as all calibration points are treated symmetrically. If instead the temperature is too low, then all the mass concentrates on only one point, meaning that $\sum_{i} (w^{(n)}_i)^2 = 1$ and hence $m_n = 1$. 

\begin{assumption}[Softmax weighting scheme]
\label{ass:weights}
Let $n$ be the number of available calibration samples, and let $\tau_n=\tau(n)$ be the temperature parameter of the \textsc{SoftMax} of equation~\ref{eq:softmax}, $m_n$ the associated effective sample size, and $\vw^{(n)}$ the output weights.
Assume:
\begin{enumerate}[label = (\roman*)]
\item $\tau_n$ is configured to slowly decrease as $n \to \infty$, so that $m_n \to \infty$; \label{subass:temperature_cooling}
\item for every $\delta>0$,
\begin{equation*}
    \sum_{s:{\|\vh_s-\vh_t\|}_2\ge\delta} w^{(n)}_s(\vh_t) \;\xrightarrow[n\to\infty]{\mathbb P}\;0.
\end{equation*}\label{subass:vanish_weights}
\end{enumerate}
\end{assumption} 
Condition \ref{subass:vanish_weights} states that for sufficiently small temperature $\tau_n$, the softmax normalization should concentrate mass on those calibration points whose states lie in a shrinking neighborhood of $\vh_t$.
Conversely, we also need these points within the shrinking neighborhood of $\vh_t$ to increase as the size of the calibration set increases. 
This can be seen as a \textit{bias-variance tradeoff}. In practical terms, the temperature $\tau_n$ must be set to a small enough value to localize weights at points similar to the query state~(reducing \textit{bias}), but at the same time it should be large enough to guarantee a good effective sample size. This requirement translates into Assumption \autoref{ass:weights}~\ref{subass:temperature_cooling} which prescribes that the temperature has to shrink at a reasonably slow rate: one that allows the number of effective neighbors $m_n$ to diverge as $n$ grows.
\begin{thm}[Consistency of the weighted empirical \gls{cdf}]
\label{thm:consistency}
Let $\widehat F_n({}\cdot{}\mid \vh_t)$ denote the conditional weighted empirical \gls{cdf} in \autoref{eq:rescp-cdf} with calibration data $\{(\vx_i, r_i)\}_{i=1}^n$ and  ${F_n(r \mid \vh_t) \coloneq \mathbb P(r_{t+H} \le r \mid \vh_{t})}$.
Under Assumptions~\ref{ass:mixing}--\ref{ass:weights}, we have for any query state ${\vh_t\in\mathbb R^{D_h}}$,
\begin{equation*}
    \sup_{r\in\mathbb R} \,\big|\widehat F_n(r\mid \vh_t)-F(r\mid \vh_{t})\big|
    \;\xrightarrow[n\to\infty]{\mathbb P}\; 0.
\end{equation*}
\end{thm}
The proof can be found in \autoref{proof:consistency}. Theorem~\ref{thm:consistency} implies the consistency of the empirical conditional quantile estimator $\widehat{Q}_\beta(\vh_t)$, defined in Equation~\ref{eq:empirical_quantile}, with respect to the true conditional quantile $Q_\beta(\vh_t)$. Finally, it is trivial to show that the asymptotic coverage of \gls{rescp} is guaranteed.
\begin{cor}[Asymptotic conditional coverage guarantee]\label{cor:asym_cond_cov}
Under Assumptions~\ref{ass:mixing}--\ref{ass:weights}, for any $\alpha \in (0,1)$ and for $n\to\infty$,
\begin{equation*}
    \mathbb{P}\left( y_{t+h} \in \hat{C}_{t}^\alpha (\hat{y}_{t+h}) \;\vert\; \vh_{t} \right) \xrightarrow[n\to\infty]{\mathbb{P}} (1-\alpha).
\end{equation*}
\end{cor}

\paragraph{Remark}
In the theoretical analysis, we model the distribution of future residuals conditioned on the state $\vh_t$, i.e., on the reservoir state at time $t$. This is clearly weaker than conditioning on the entire history $\vx_{\leq t}$. Recalling \autoref{eq:rescp-cdf}, whether or not $\widehat{F}(r \mid \vh_t)$ is a good approximation of ${F}(r \mid \vx_{\leq t})$ entirely depends on the ability of the \gls{esn} to encode all relevant information from the past in its state.  In particular, if the state representation captures all the relevant information, then \gls{rescp} provides asymptotic conditional coverage given such a history. Conversely, in the extreme case where every sequence is mapped by the reservoir into the same uninformative state, \gls{rescp} would simply provide marginal coverage, as one would expect. This introduces an additional bias-variance tradeoff. For example, the reservoir mapping all trajectories to similar states would result in small variance, but would have small discriminative power~(large bias). We refer the reader to the rich body of literature on the expressiveness of \glspl{esn}, particularly regarding their universal approximation properties~\citep{grigoryeva2018echo, grigoryeva2018universal, li2025universality} and their effectiveness in extracting meaningful representations from time series data~\citep{bianchi2020reservoir}. In particular, we refer to \citep{lukovsevivcius2012practical} for practical guidelines on how to set up \glspl{esn} effectively.

\subsubsection{Time-dependent weights}
\gls{rescp}, by default, compares the query state with \textit{any} time step from the calibration set. Note that state representations do not include any positional encoding that accounts for how far back in time a certain sample is, treating recent and distant residuals alike. 
This is not problematic for time-invariant and stable processes, which we assumed in \ref{sec:theory} to provide coverage guarantees. 
To deal, instead, with distribution shifts, we 1) update the calibration set over time, keeping its size $N$ fixed using a first-in-first-out approach, and 2) make the weights time-dependent, following an approach similar to NexCP~\citep{barber2023conformal}. 
In particular, we condition the weights on the distance between time steps as
\begin{equation}
    w_i(\bm{h}_t, t) = \gamma (\Delta(t, i)) w_i(\bm{h}_t),
\end{equation}
where $\gamma : \sN \to \sR$ is a discount function that maps the distance between time steps $\Delta(t, i)$ to a decay factor. 
The discount schedule can be chosen in different ways, e.g., exponential like NexCP~\citep{barber2023conformal} or linear. 
In our settings, we observed that a linear decay--i.e., $\gamma (\Delta(t, i)) \coloneq 1/\Delta(t, i)$--allowed to keep \glspl{pi} up-to-date without reducing too much the effective sample size. 
This is particularly important if assumption \ref{ass:mixing} is violated and the underlying dynamics change over time. 
It is worth noting that \gls{rescp} is a non-parametric approach without learnable parameters that is much more robust to non-stationarity compared to approaches that train a model~\citep{auer2023conformal, cini2025relational}, which must necessarily be updated when time-invariance is lost.

\subsection{Reservoir Conformal Quantile Regression}\label{sec:rescqr}

Usually \gls{cp} methods operate directly on conformal scores. However, in some case it might be beneficial to consider exogenous inputs~(covariates).
Adding exogenous variables as input to the reservoir can affect the internal dynamics of the network and, thus, its states. Depending on the characteristics of specific exogenous variables, they can harm the effectiveness of the reservoir in localizing prediction w.r.t.\ the dynamics of the target variable. Since the network is not trained, adjusting to the relevance and characteristics of the covariates is not possible with the completely unsupervised approach discussed in \autoref{sec:rescp}. 
As an alternative to \gls{rescp} in these scenarios, we consider a variant of our original approach, called \gls{rescqr}, which can account for exogenous inputs by relying on quantile regression. In practice, we use a linear readout to map the state at each time step to a set of quantiles of interest.  
Specifically, we use states $\vh_{1:T-H}$ and residuals $r_{H+1:T}$ in the calibration set to fit a linear model $\widehat{Q}_\alpha(\vh_t)$  as a quantile predictor. 
To do so, we train the readout to minimize the pinball loss for the target quantiles $\{\beta_1, \dots, \beta_M\}$ of the conformity scores:
\begin{align*}
    \mathcal{L}_{\beta_i}(\widehat{Q}_{\beta_i}(\vh_t), r) = 
    \begin{cases}
        (1-\beta_i)(\widehat{Q}_{\beta_i}(\vh_t) - r), & \widehat{Q}_{\beta_i}(\vh_t) \geq r \\
        \beta_i(r - \widehat{Q}_{\beta_i}(\vh_t)), & \widehat{Q}_{\beta_i}(\vh_t) < r \\
    \end{cases}
\end{align*}
We can then construct \glspl{pi} as discussed in \autoref{sec:rescp}. As we will show in \autoref{sec:experiments}, \gls{rescqr} provides a practical methods approach that can work well in scenarios where the availability of enough calibration data allows for training the readout and informative exogenous variables are available. 

\input{multi_dataset_table}

\section{Experiments}\label{sec:experiments} 
We compare the performances of \gls{rescp} against state-of-the-art conformal prediction baselines on time series data coming from several applications.

\paragraph{Datasets and baselines}
We evaluated our method across four datasets. 
1) The \textbf{Solar} dataset comes from the US National Solar Radiation Database~\citep{sengupta2018the}. 
We used the dataset containing 50 time series from different locations over a period of 3 years, as done in previous work~\citep{auer2023conformal}. 
2) The \textbf{Beijing} dataset contains air quality measurements taken over a period of 4 years from 12 locations in the city of Beijing, China~\citep{zhang2017cautionary}. 
3) The \textbf{Exchange} dataset consists of a collection of the daily exchange rates of eight countries from 1990 to 2016~\citep{lai2017modeling}. 
4) \textbf{ACEA} contains electricity consumption data coming from the backbone of the energy supply network in the city of Rome~\citep{bianchi2015shortterm}. 
More details are reported in \autoref{app:datasets}. 
We compare \gls{rescp} and its \gls{rescqr} variant against the competitors presented in \autoref{sec:background}: 
1) vanilla \textbf{\gls{scp}}~\citep{vovk2005algorithmic}, 2) \textbf{\gls{nexcp}}~\citep{barber2023conformal}, 3) \textbf{\gls{scpi}}~\citep{chen2023sequential}, and 4) \textbf{\gls{hopcpt}}~\citep{auer2023conformal}.
We also include 5) \textbf{\gls{cornn}}, an \gls{rnn} with a multi-dimensional output trained to perform quantile regression on the calibration set using the pinball loss. The architecture is analogous to the model called \textsc{CoRNN} introduced in \cite{cini2025relational}~(more details in \autoref{app:hyperparameters}).
For \gls{rescp}, we use time-dependent weights~(with a linear decay schedule) and the cosine similarity between reservoir states as the similarity score. 

\paragraph{Experimental setup and evaluation metrics} 
\label{sec:exp_setup}
In our experiments, we adopted a $40\%/40\%/20\%$ split for training, calibration, and testing sets, respectively. 
As base models, we consider three different point predictors: a simple \gls{rnn} with gated recurrent cells~\citep{cho2014learning}, a decoder-only Transformer~\citep{vaswani2017attention}, and an ARIMA model~\citep{box1970time}. 
After training, we evaluated each model on the calibration set and saved the residuals, which we then used in all the baselines to compute the \glspl{pi}. 
As evaluation metrics, we considered the $\Delta\text{Cov}$, i.e., the difference between the specified confidence level $1-\alpha$ and the achieved coverage on the test set, the width of the \glspl{pi}, and the Winkler score~\citep{winkler1972decision}, which penalizes the \gls{pi} width whenever the observed value falls outside the computed interval, with the penalty scaled proportionally to the magnitude of the deviation. 
Model selection for all methods is done by minimizing the Winkler score over a validation set, except for \gls{hopcpt}, which follows a custom procedure~\citep{auer2023conformal}. 
More details on model selection for our methods are reported in \autoref{app:hyperparameters}.
For \gls{rescp}, when approximating the quantiles via Monte Carlo sampling, we set the number of observations used for calibration as the sample size. 
As mentioned in \autoref{sec:rescp}, we account for skewed distributions by building \glspl{pi} as in \autoref{eq:rescw_int_beta}. The optimal $\beta$ is chosen from $100$ linearly spaced values between $0$ to $\alpha$.

\subsection{Results} 


\input{multi_simplified_runtime}

\input{imgs/multimethod_calib_plot}

Results across the datasets and base models are summarized in \autoref{tab:multi_results_combined}, where \gls{cp} methods are grouped based on whether they rely on learning or not. Additional results for different miscoverage levels can be found in \autoref{app:calib_curves} (see \autoref{tab:results_alpha_0.05} and \autoref{tab:results_alpha_0.15}).
\gls{rescp} achieves competitive performance across all settings and datasets, both in terms of coverage and Winkler score, while remaining highly scalable. 
Notably, \glspl{rescp} outperforms HopCPT in almost all scenarios. 
This is particularly remarkable since both methods use weighted empirical distributions to model uncertainty, but HopCPT requires training an attention-based architecture end-to-end, while \glspl{rescp} does not perform any form of training. 
Moreover, similarly to NexCP and differently from most of the other methods, \gls{rescp} provides approximately valid coverage across all scenarios, but with large improvements in \gls{pi} width, which is reduced up to 60\%. 
The \gls{cornn} baseline achieves strong performance on the large datasets with informative exogenous variables, outperforming all methods in Solar, but fails in achieving good results in the smaller datasets, such as ACEA and Exchange. 
Moreover, on those datasets, trainable methods obtain worse performance, likely due to distribution changes that require the trained models to be updated over time. 
\gls{rescqr} achieves competitive performance against the competitors, including more complex models such as \gls{cornn} in Solar, while being more scalable. 
All other baselines obtain good coverage in most settings, but produce prediction intervals that are much more conservative. 
The runtime of each \gls{cp} method is reported in \autoref{tab:simplified_runtime_rnn} and shows the advantage of \gls{rescp} in terms of scalability against methods that require fitting a model. 
Moreover, as \gls{rescp} does not need centralized training on a GPU: it can be easily scaled in a distributed setting.

\paragraph{Additional experiments and ablation studies} \autoref{fig:calib_grid} presents calibration curves for \gls{rescp} and \gls{rescqr} and the \gls{cornn} and \gls{nexcp} baselines, across all datasets and using an \gls{rnn} as the base point predictor. The curves show that \gls{rescp} provides accurate estimates at all the considered coverage levels. Note that while \gls{nexcp} is well calibrated, it  produces much wider intervals~(see \autoref{tab:multi_results_combined}, \autoref{tab:results_alpha_0.05} and \autoref{tab:results_alpha_0.15}). 
Additional results are provided in \autoref{app:calib_curves}, including an experiment of \gls{rescp} on a non-stationary synthetic dataset.
In \autoref{tab:ablation_study_rnn}, we also report an ablation study of \gls{rescp}. In particular, we evaluated the impact of 1) removing  time-dependent weights~(\textbf{No decay}), 2) using all the available samples for calibration rather than only the most recent ones~(\textbf{No window}), and 3) the combination of the previous two ablations~(\textbf{No window, no decay}). The results clearly show the impcat of the proposed designs. More comprehensive ablation results for all base models are presented in \autoref{app:ablation}, which also include 
an ablation study on the use of exogenous variables.
Finally, \autoref{app:sensitivity} reports additional results on the sensitivity of \gls{rescp} to different hyperparameter configurations. 

\input{ablation_table_rnn_only}

\section{Conclusion and future works}\label{sec:conclusion} 

We introduced \acrfull{rescp}, a simple, fast, and effective method for uncertainty quantification in time series
forecasting that provides accurate uncertainty estimation across several tasks. 
We showed that \gls{rescp} is principled and theoretically sound.
In particular, assuming strong mixing conditions and reasonable regularity conditions, we demonstrated that \gls{rescp} provides asymptotically valid conditional coverage guarantees. 
Results on a diverse suite of benchmarks show that \gls{rescp} achieves state-of-the-art performance while being more scalable and computationally efficient than its competitors. 
There are several directions for future work. In particular, it would be interesting to explore extensions to modeling joint distributions in multistep forecasting~\citep{sun2024copula}, treat multidimensional time series~\citep{feldman2023calibrated, xu2024conformal}, and spatiotemporal data~\citep{cini2025relational}.

\input{acknowledgements}

%% file: imgs/method_img.tex
\begin{figure}[t]
    \centering
    \includegraphics[width=0.7\linewidth]{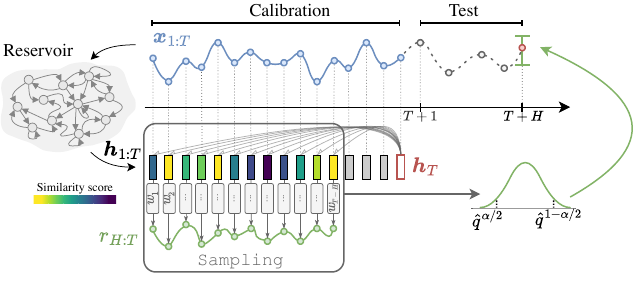}
    \vspace{-.3cm}
    \caption{Let $\{\bm{x}_t\}_{t=1}^T$ and $\{{r}_t\}_{t=1}^T$ be the time series and residuals of the calibration set, respectively. 
    Each sequence $\bm{x}_{1:t}$ generates the state of the reservoir $\bm{h}_t$.
    The last state $\bm{h}_T$ is the query state for the prediction of $\hat{y}_{T+H}$.
    We compute similarity scores between $\bm{h}_T$ and the calibration states $\{\bm{h}_t\}_{t=1}^{T-H}$, which are used to reweight with $\{{w}_t\}_{t=1}^{T-H}$ associated residuals $\{{r}_t\}_{t=H}^T$ by resampling them. 
    Quantiles are computed from the sampled residuals and used to build the \gls{pi} for $\hat{y}_{T+H}$. 
    } 
    \label{fig:overview}
\end{figure}

%% file: multi_dataset_table.tex
\begin{table*}[t]
\centering
\small
\caption{Performance comparison for $\alpha=0.1$. 
$\Delta$Cov values are color-coded for undercoverage cases: \textcolor{tabolive}{yellow} (1-2\%), \textcolor{taborange}{orange} (2-4\%), \textcolor{tabred}{red} ($>$4\%). 
The top three Winkler scores for each scenario are highlighted: \textbf{\underline{bold+underlined}} (1st), \underline{underlined} (2nd), \dotuline{dot-underlined} (3rd). }
\setlength{\tabcolsep}{2.5pt}
\setlength{\aboverulesep}{0pt}
\setlength{\belowrulesep}{0pt}
\renewcommand{\arraystretch}{1.05}
\resizebox{\textwidth}{!}{

\begin{tabular}{@{}l|l|l|cccc|ccc@{}}
\multicolumn{3}{c|}{}  & \multicolumn{4}{c|}{\textit{Learning}} & \multicolumn{3}{c}{\textit{Non-learning}} \\
\multicolumn{2}{c}{}    & Metric        & \gls{scpi} & \gls{hopcpt} & \gls{cornn} & \textbf{\gls{rescqr}} & \gls{scp} & \gls{nexcp} & \textbf{\gls{rescp}} \\
\midrule[1.5pt]
\multirow{9}{*}{\rotatebox{90}{Solar}}
 & \multirow{3}{*}{\rotatebox{90}{RNN}} & $\Delta$Cov   & 0.05{\tiny$\pm$0.17} & \textcolor{tabolive}{-1.64{\tiny$\pm$1.18}} & -0.26{\tiny$\pm$0.92} & \textcolor{tabolive}{-1.10{\tiny$\pm$0.91}} & 0.37 & 1.46 & 0.74{\tiny$\pm$0.24} \\
 &                             & PI-Width      & 70.41{\tiny$\pm$1.73} & 60.49{\tiny$\pm$2.10} & 55.74{\tiny$\pm$0.98} & 59.99{\tiny$\pm$1.72} & 79.15 & 89.56 & 62.25{\tiny$\pm$0.75} \\
 &                             & Winkler       & 151.50{\tiny$\pm$1.18} & 112.46{\tiny$\pm$9.34} & \textbf{\underline{78.42{\tiny$\pm$0.20}}} & \underline{82.76{\tiny$\pm$0.26}} & 171.41 & 164.96 & \dotuline{104.24{\tiny$\pm$0.79}} \\
\cmidrule[0.2pt]{3-10}
 & \multirow{3}{*}{\rotatebox{90}{Transf}} & $\Delta$Cov   & -0.16{\tiny$\pm$0.46} & 1.32{\tiny$\pm$0.62} & 1.37{\tiny$\pm$2.09} & \textcolor{taborange}{-3.51{\tiny$\pm$16.26}} & 0.35 & 1.46 & 3.09{\tiny$\pm$0.35} \\
 &                             & PI-Width      & 71.36{\tiny$\pm$3.21} & 61.49{\tiny$\pm$1.74} & 55.70{\tiny$\pm$0.95} & 59.56{\tiny$\pm$1.59} & 79.03 & 89.21 & 63.34{\tiny$\pm$1.11} \\
 &                             & Winkler       & 152.85{\tiny$\pm$0.67} & 107.59{\tiny$\pm$6.07} & \textbf{\underline{77.61{\tiny$\pm$0.25}}} & \underline{82.16{\tiny$\pm$0.32}} & 169.64 & 163.34 & \dotuline{103.13{\tiny$\pm$0.58}} \\
\cmidrule[0.2pt]{3-10}
 & \multirow{3}{*}{\rotatebox{90}{ARIMA}} & $\Delta$Cov   & 0.51{\tiny$\pm$0.36} & 2.88{\tiny$\pm$2.13} & -0.41{\tiny$\pm$0.42} & \textcolor{taborange}{-2.03{\tiny$\pm$0.62}} & 0.16 & 1.76 & 0.68{\tiny$\pm$0.95} \\
 &                             & PI-Width      & 91.71{\tiny$\pm$1.18} & 143.32{\tiny$\pm$7.36} & 59.17{\tiny$\pm$0.41} & 66.19{\tiny$\pm$0.81} & 124.19 & 137.53 & 77.17{\tiny$\pm$2.07} \\
 &                             & Winkler       & 148.86{\tiny$\pm$0.34} & 173.49{\tiny$\pm$5.15} & \textbf{\underline{77.34{\tiny$\pm$0.30}}} & \underline{85.38{\tiny$\pm$0.45}} & 215.77 & 207.58 & \dotuline{110.38{\tiny$\pm$4.03}} \\
\cmidrule[1.5pt]{1-10}
\multirow{9}{*}{\rotatebox{90}{Beijing}}
 & \multirow{3}{*}{\rotatebox{90}{RNN}} & $\Delta$Cov   & \textcolor{tabolive}{-1.73{\tiny$\pm$0.67}} & \textcolor{tabred}{-5.18{\tiny$\pm$12.67}} & \textcolor{tabolive}{-1.86{\tiny$\pm$2.16}} & \textcolor{tabolive}{-1.21{\tiny$\pm$1.65}} & -0.32 & 0.00 & -0.70{\tiny$\pm$0.77} \\
 &                             & PI-Width      & 67.93{\tiny$\pm$1.43} & 68.47{\tiny$\pm$13.30} & 61.71{\tiny$\pm$4.91} & 65.53{\tiny$\pm$4.09} & 67.99 & 69.79 & 65.96{\tiny$\pm$2.50} \\
 &                             & Winkler       & 124.51{\tiny$\pm$2.25} & 140.50{\tiny$\pm$43.64} & \textbf{\underline{104.03{\tiny$\pm$0.99}}} & \underline{105.43{\tiny$\pm$0.85}} & 126.41 & 124.10 & \dotuline{106.07{\tiny$\pm$0.47}} \\
\cmidrule[0.2pt]{3-10}
 & \multirow{3}{*}{\rotatebox{90}{Transf}} & $\Delta$Cov   & -0.98{\tiny$\pm$1.04} & \textcolor{tabred}{-8.05{\tiny$\pm$16.41}} & \textcolor{tabolive}{-1.07{\tiny$\pm$0.69}} & \textcolor{tabolive}{-1.43{\tiny$\pm$1.10}} & -0.41 & 0.03 & -0.49{\tiny$\pm$0.59} \\
 &                             & PI-Width      & 69.96{\tiny$\pm$4.53} & 61.76{\tiny$\pm$14.39} & 62.41{\tiny$\pm$1.66} & 64.41{\tiny$\pm$2.72} & 67.46 & 69.64 & 64.06{\tiny$\pm$1.74} \\
 &                             & Winkler       & 125.41{\tiny$\pm$1.43} & 140.30{\tiny$\pm$36.01} & \textbf{\underline{102.81{\tiny$\pm$0.44}}} & \dotuline{105.97{\tiny$\pm$1.21}} & 126.63 & 124.35 & \underline{103.64{\tiny$\pm$0.21}} \\
\cmidrule[0.2pt]{3-10}
 & \multirow{3}{*}{\rotatebox{90}{ARIMA}} & $\Delta$Cov   & -0.23{\tiny$\pm$0.40} & \textcolor{tabolive}{-1.37{\tiny$\pm$0.26}} & \textcolor{tabolive}{-1.54{\tiny$\pm$0.77}} & \textcolor{tabolive}{-1.42{\tiny$\pm$1.15}} & -0.24 & -0.16 & 0.63{\tiny$\pm$0.22} \\
 &                             & PI-Width      & 74.68{\tiny$\pm$1.21} & 67.78{\tiny$\pm$0.50} & 61.80{\tiny$\pm$1.68} & 66.01{\tiny$\pm$3.01} & 75.72 & 76.45 & 70.43{\tiny$\pm$0.86} \\
 &                             & Winkler       & 130.59{\tiny$\pm$0.59} & 122.48{\tiny$\pm$4.36} & \textbf{\underline{101.84{\tiny$\pm$0.67}}} & \underline{107.20{\tiny$\pm$1.21}} & 135.07 & 132.03 & \dotuline{108.75{\tiny$\pm$0.31}} \\
\cmidrule[1.5pt]{1-10}
\multirow{9}{*}{\rotatebox{90}{Exchange}}
 & \multirow{3}{*}{\rotatebox{90}{RNN}} & $\Delta$Cov   & 2.98{\tiny$\pm$0.65} & 2.75{\tiny$\pm$0.08} & \textcolor{tabolive}{-1.07{\tiny$\pm$2.52}} & 3.18{\tiny$\pm$1.25} & 2.29 & 1.64 & 1.13{\tiny$\pm$0.27} \\
 &                             & PI-Width      & 0.0241{\tiny$\pm$0.0007} & 0.0404{\tiny$\pm$0.0001} & 0.0341{\tiny$\pm$0.0018} & 0.0383{\tiny$\pm$0.0008} & 0.0444 & 0.0405 & 0.0210{\tiny$\pm$0.0001} \\
 &                             & Winkler       & \underline{0.0287{\tiny$\pm$0.0007}} & 0.0482{\tiny$\pm$0.0001} & \dotuline{0.0461{\tiny$\pm$0.0005}} & 0.0464{\tiny$\pm$0.0005} & 0.0517 & 0.0492 & \textbf{\underline{0.0264{\tiny$\pm$0.0002}}} \\
\cmidrule[0.2pt]{3-10}
 & \multirow{3}{*}{\rotatebox{90}{Transf}} & $\Delta$Cov   & 4.44{\tiny$\pm$0.35} & 2.98{\tiny$\pm$0.07} & -0.57{\tiny$\pm$1.58} & 0.82{\tiny$\pm$1.89} & 4.57 & 3.25 & 1.46{\tiny$\pm$0.18} \\
 &                             & PI-Width      & 0.0255{\tiny$\pm$0.0005} & 0.0399{\tiny$\pm$0.0001} & 0.0337{\tiny$\pm$0.0016} & 0.0365{\tiny$\pm$0.0013} & 0.0544 & 0.0509 & 0.0229{\tiny$\pm$0.0001} \\
 &                             & Winkler       & \underline{0.0300{\tiny$\pm$0.0005}} & 0.0479{\tiny$\pm$0.0001} & 0.0480{\tiny$\pm$0.0009} & \dotuline{0.0475{\tiny$\pm$0.0008}} & 0.0620 & 0.0602 & \textbf{\underline{0.0294{\tiny$\pm$0.0001}}} \\
\cmidrule[0.2pt]{3-10}
 & \multirow{3}{*}{\rotatebox{90}{ARIMA}} & $\Delta$Cov   & 3.49{\tiny$\pm$0.41} & 2.07{\tiny$\pm$0.08} & \textcolor{tabolive}{-1.22{\tiny$\pm$1.78}} & 0.68{\tiny$\pm$1.58} & 3.08 & 2.13 & 0.38{\tiny$\pm$0.41} \\
 &                             & PI-Width      & 0.0242{\tiny$\pm$0.0006} & 0.0379{\tiny$\pm$0.0000} & 0.0330{\tiny$\pm$0.0007} & 0.0351{\tiny$\pm$0.0009} & 0.0387 & 0.0356 & 0.0207{\tiny$\pm$0.0001} \\
 &                             & Winkler       & \underline{0.0289{\tiny$\pm$0.0003}} & 0.0456{\tiny$\pm$0.0001} & 0.0455{\tiny$\pm$0.0003} & 0.0455{\tiny$\pm$0.0006} & 0.0462 & \dotuline{0.0447} & \textbf{\underline{0.0268{\tiny$\pm$0.0001}}} \\
\cmidrule[1.5pt]{1-10}
\multirow{9}{*}{\rotatebox{90}{ACEA}}
 & \multirow{3}{*}{\rotatebox{90}{RNN}} & $\Delta$Cov   & -0.78{\tiny$\pm$1.88} & \textcolor{taborange}{-2.18{\tiny$\pm$0.00}} & \textcolor{tabred}{-12.37{\tiny$\pm$8.98}} & \textcolor{tabred}{-18.86{\tiny$\pm$7.44}} & -0.99 & -0.33 & 1.56{\tiny$\pm$0.62} \\
 &                             & PI-Width      & 8.99{\tiny$\pm$0.68} & 18.90{\tiny$\pm$0.00} & 15.86{\tiny$\pm$1.99} & 15.23{\tiny$\pm$1.96} & 19.63 & 20.15 & 9.61{\tiny$\pm$0.26} \\
 &                             & Winkler       & \underline{14.27{\tiny$\pm$0.19}} & 27.56{\tiny$\pm$0.00} & 32.61{\tiny$\pm$5.69} & 34.61{\tiny$\pm$3.53} & 27.60 & \dotuline{26.83} & \textbf{\underline{12.91{\tiny$\pm$0.23}}} \\
\cmidrule[0.2pt]{3-10}
 & \multirow{3}{*}{\rotatebox{90}{Transf}} & $\Delta$Cov   & \textcolor{tabolive}{-1.41{\tiny$\pm$1.29}} & \textcolor{taborange}{-2.51{\tiny$\pm$0.00}} & \textcolor{tabred}{-13.35{\tiny$\pm$9.85}} & \textcolor{tabred}{-26.92{\tiny$\pm$7.68}} & \textcolor{tabred}{-5.52} & -0.45 & 3.54{\tiny$\pm$0.32} \\
 &                             & PI-Width      & 9.10{\tiny$\pm$0.23} & 18.29{\tiny$\pm$0.00} & 14.82{\tiny$\pm$2.02} & 13.20{\tiny$\pm$1.47} & 16.53 & 20.20 & 10.10{\tiny$\pm$0.16} \\
 &                             & Winkler       & \underline{14.58{\tiny$\pm$0.36}} & 27.53{\tiny$\pm$0.00} & 33.47{\tiny$\pm$7.18} & 39.98{\tiny$\pm$4.27} & 29.24 & \dotuline{27.47} & \textbf{\underline{12.90{\tiny$\pm$0.16}}} \\
\cmidrule[0.2pt]{3-10}
 & \multirow{3}{*}{\rotatebox{90}{ARIMA}} & $\Delta$Cov   & 1.41{\tiny$\pm$0.90} & \textcolor{taborange}{-3.58{\tiny$\pm$0.00}} & \textcolor{tabred}{-29.35{\tiny$\pm$11.01}} & \textcolor{tabred}{-27.10{\tiny$\pm$8.65}} & -0.75 & -0.40 & 5.02{\tiny$\pm$0.40} \\
 &                             & PI-Width      & 12.46{\tiny$\pm$0.35} & 34.84{\tiny$\pm$0.00} & 18.16{\tiny$\pm$2.43} & 17.39{\tiny$\pm$2.20} & 38.13 & 36.08 & 13.63{\tiny$\pm$0.55} \\
 &                             & Winkler       & \underline{17.36{\tiny$\pm$0.13}} & 44.69{\tiny$\pm$0.00} & 53.89{\tiny$\pm$9.33} & 48.49{\tiny$\pm$6.89} & 45.99 & \dotuline{43.70} & \textbf{\underline{16.21{\tiny$\pm$0.53}}} \\
\bottomrule[1.5pt]
\end{tabular}
}
\label{tab:multi_results_combined}
\end{table*}

%% file: multi_simplified_runtime.tex
\begin{table*}[t]
\centering
\small
\caption{Runtime (in seconds) using the \gls{rnn} point forecasting baseline.
}

\setlength{\tabcolsep}{2.5pt}
\setlength{\aboverulesep}{0pt}
\setlength{\belowrulesep}{0pt}
\renewcommand{\arraystretch}{1.05}

\begin{tabular}{@{}l|cccc|ccc@{}}
\multicolumn{1}{c|}{}  & \multicolumn{4}{c|}{\textit{Learning}} & \multicolumn{3}{c}{\textit{Non-learning}} \\
Dataset & \gls{scpi} & \gls{hopcpt} & \gls{cornn} & \textbf{\gls{rescqr}} & \gls{scp} & \gls{nexcp} & \textbf{\gls{rescp}} \\
\midrule[1.5pt]
Solar & 1039.8{\tiny$\pm$2} & 4574.6{\tiny$\pm$1356.5} & 172.4{\tiny$\pm$17} & 82.1{\tiny$\pm$6.7} & 18 & 66 & 52.9{\tiny$\pm$14.7} \\
\cmidrule[0.2pt]{1-8}
Beijing & 351.4{\tiny$\pm$2.1} & 1838.5{\tiny$\pm$202.4} & 81.6{\tiny$\pm$3} & 46{\tiny$\pm$2.1} & 9 & 29 & 34.6{\tiny$\pm$1.4} \\
\cmidrule[0.2pt]{1-8}
Exchange & 50.6{\tiny$\pm$0.5} & 318{\tiny$\pm$0} & 36.9{\tiny$\pm$1.2} & 16{\tiny$\pm$0.9} & 2 & 2 & 6.5{\tiny$\pm$0.5} \\
\cmidrule[0.2pt]{1-8}
ACEA & 227.6{\tiny$\pm$5.8} & 2262.8{\tiny$\pm$918.9} & 95.3{\tiny$\pm$9.5} & 56.6{\tiny$\pm$2.9} & 7 & 57 & 70.5{\tiny$\pm$0.5} \\
\bottomrule[1.5pt]
\end{tabular}

\label{tab:simplified_runtime_rnn}
\end{table*}

%% file: imgs/multimethod_calib_plot.tex
\begin{figure}[b]
    \centering
    \includegraphics[width=0.9\linewidth, trim=50 20 50 30]{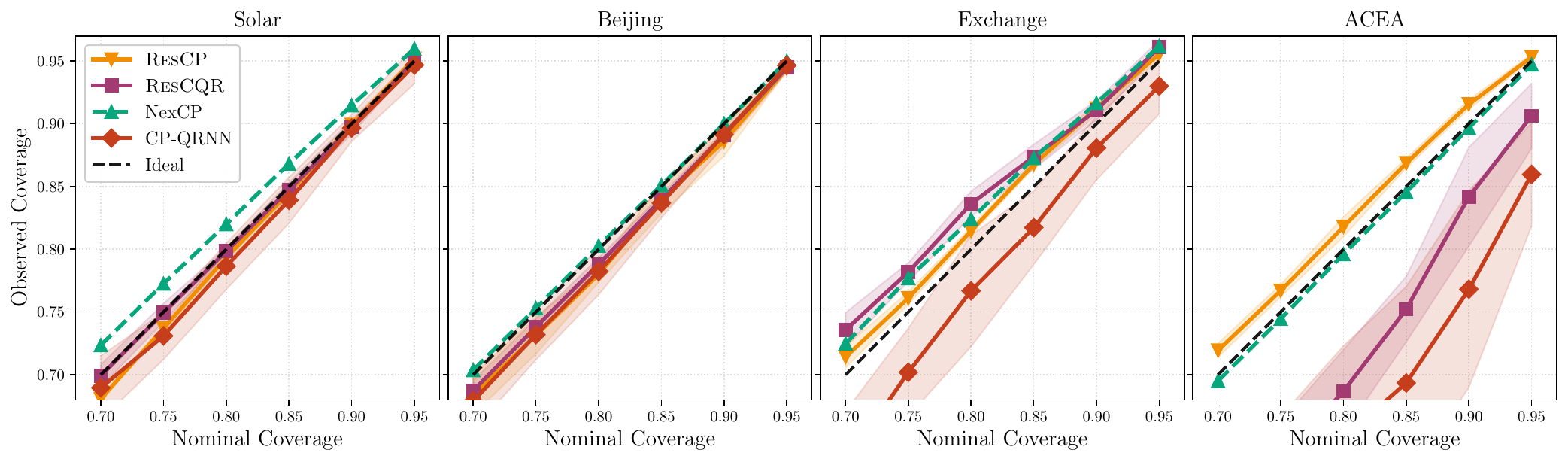}
    \caption{{Calibration curves of \gls{rescp}, \gls{rescqr}, \gls{cornn} and \gls{nexcp} over all datasets with the RNN baseline.}}
    \label{fig:calib_grid}
\end{figure}

%% file: ablation_table_rnn_only.tex
\begin{table}[t]
\centering
\caption{{Ablation of \textbf{\gls{rescp}}, \gls{rescp} without decay, \gls{rescp} without sliding window, and the combination of the two, using the \gls{rnn} base model. Best Winkler scores are shown in bold.}}
\setlength{\tabcolsep}{3pt}
\setlength{\aboverulesep}{0pt}
\setlength{\belowrulesep}{0pt}
\renewcommand{\arraystretch}{1.05}

\begin{tabular}{@{}l|l|c|ccc@{}}
\multicolumn{1}{c}{} & Metric & \textbf{\gls{rescp}} & {\scriptsize No decay} & {\scriptsize No window} & {\scriptsize No window, no decay} \\
\midrule[1.5pt]
\multirow{3}{*}{\rotatebox{90}{Solar}} & $\Delta$Cov & 0.74{\tiny$\pm$0.24} & -0.10{\tiny$\pm$0.26} & 0.89{\tiny$\pm$0.20} & 1.70{\tiny$\pm$0.10} \\
 & PI-Width & 62.25{\tiny$\pm$0.75} & 60.59{\tiny$\pm$0.84} & 59.81{\tiny$\pm$0.87} & 60.90{\tiny$\pm$0.30} \\
 & Winkler & \textbf{104.24{\tiny$\pm$0.79}} & 107.46{\tiny$\pm$0.49} & 104.70{\tiny$\pm$0.20} & 104.46{\tiny$\pm$0.40} \\
\midrule
\multirow{3}{*}{\rotatebox{90}{Beijing}} & $\Delta$Cov & -0.70{\tiny$\pm$0.77} & -1.64{\tiny$\pm$0.61} & 0.10{\tiny$\pm$1.36} & 1.64{\tiny$\pm$0.88} \\
 & PI-Width & 65.96{\tiny$\pm$2.50} & 64.04{\tiny$\pm$1.69} & 64.91{\tiny$\pm$3.36} & 69.53{\tiny$\pm$3.02} \\
 & Winkler & 106.07{\tiny$\pm$0.47} & 114.85{\tiny$\pm$0.29} & \textbf{97.99{\tiny$\pm$0.31}} & 98.25{\tiny$\pm$0.66} \\
\midrule
\multirow{3}{*}{\rotatebox{90}{Exch.}} & $\Delta$Cov & 1.13{\tiny$\pm$0.27} & 2.01{\tiny$\pm$0.19} & 4.19{\tiny$\pm$0.05} & 4.20{\tiny$\pm$0.12} \\
 & PI-Width & 0.0210{\tiny$\pm$0.0001} & 0.0219{\tiny$\pm$0.0001} & 0.0249{\tiny$\pm$0.0002} & 0.0254{\tiny$\pm$0.0003} \\
 & Winkler & \textbf{0.0264{\tiny$\pm$0.0002}} & 0.0269{\tiny$\pm$0.0002} & 0.0284{\tiny$\pm$0.0001} & 0.0291{\tiny$\pm$0.0002} \\
\midrule
\multirow{3}{*}{\rotatebox{90}{ACEA}} & $\Delta$Cov & 1.56{\tiny$\pm$0.62} & 2.79{\tiny$\pm$0.45} & 5.34{\tiny$\pm$0.38} & 4.96{\tiny$\pm$0.63} \\
 & PI-Width & 9.61{\tiny$\pm$0.26} & 10.15{\tiny$\pm$0.09} & 11.88{\tiny$\pm$0.09} & 12.15{\tiny$\pm$0.17} \\
 & Winkler & \textbf{12.91{\tiny$\pm$0.23}} & 13.41{\tiny$\pm$0.08} & 14.80{\tiny$\pm$0.17} & 15.25{\tiny$\pm$0.40} \\
\bottomrule
\end{tabular}

\label{tab:ablation_study_rnn}
\end{table}

%% file: acknowledgements.tex
\subsubsection*{Acknowledgments}
This work is supported by the Research Council of Norway through \textit{RELAY:~Relational Deep Learning for Energy Analytics} (project no. 345017) and through its Centre of Excellence \textit{Integreat - The Norwegian Centre for knowledge-driven machine learning} (project no. 332645), the Swiss National Science Foundation grant no.~225351~(\textit{Relational Deep Learning for Reliable Time Series Forecasting at Scale}), the EPSRC Turing AI World-Leading Research Fellowship No. EP/X040062/1 and EPSRC AI Hub No. EP/Y028872/1.
The authors wish to thank Nvidia Corporation for donating some of the GPUs used in this project and Carlo Abate for the help in checking the proof.

%% file: appendix.tex
\section{Proofs for the asymptotic conditional coverage}\label{proof:consistency}

We start the proof with the following lemma.

\begin{lem}[Regularity of the conditional CDF]
\label{lem:regularity}
Under Assumptions~\ref{ass:esn}--\ref{ass:continuity}, the conditional \gls{cdf} $F(r\mid \vh_t)$
is continuous in $\vh_t \in\mathbb R^{d_h}$, uniformly in $r\in\mathbb R$. 
\end{lem}

\begin{proof}
By Assumption~\ref{ass:esn}, the map $\vx_{\leq t}\mapsto \vh_t$ is Lipschitz
with respect to $d_{\mathrm{fm}}$. By Assumption~\ref{ass:continuity}, the map
$\vx_{\leq t}\mapsto F(r\mid \vx_{\leq t})$ is continuous.
Therefore the composition ${\vx_{\leq t}\mapsto \vh_t\mapsto F(r\mid \vh_t)}$
is continuous in $\vh_t$. 
\end{proof}

We then continue with the proof of Theorem~\ref{thm:consistency}, which we start by proving point wise convergence. 

\begin{proof}[Proof of Theorem~\ref{thm:consistency}]
Define
\begin{equation*}
    \bar{F}_n(r \mid\vh_t) \coloneqq \sum_{i=1}^n w^{(n)}_i(\vh_t)F(r\mid\vh_i).
\end{equation*}
Decompose
\begin{equation*}
    \widehat F_n(r\mid \vh_t)-F(r\mid \vh_t)
    = \underbrace{(\widehat{F}_n(r \mid\vh_t) - \bar{F}_n(r \mid\vh_t))}_{\mathrm{(I)}}
    + \underbrace{(\bar{F}_n(r \mid\vh_t) - {F}(r \mid\vh_t))}_{\mathrm{(II)}}.
\end{equation*}
We handle the two terms separately and show both going uniformly to 0 in probability.

{
\paragraph{Term (I):}
\begin{equation*}
    \widehat{F}_n(r \mid\vh_t) - \bar{F}_n(r \mid\vh_t) = \sum_{i=1}^n w^{(n)}_i(\vh_t)\big(\mathds{1}\{r_{i+H} \leq r \} - F(r\mid \vh_i)\big) = \sum_{i=1}^n w^{(n)}_i(\vh_t)Y_i(r) = S_n(r)
\end{equation*}
with $Y_i(r) = \mathds{1}\{r_{i+H} \leq r \} - F(r\mid \vh_i)$.
We can see that $\mathbb{E}\left[ Y_i(r) \mid \vh_i \right] = 0$, since
\begin{align*}
    \mathbb{E}[\mathds{1}\{ r_i \leq r\} \mid \vh_i] &= 1\cdot\mathbb{P}(r_i \leq r \mid \vh_i) + 0 \cdot \mathbb{P}(r_i > r \mid \vh_i) = \mathbb{P}(r_i \leq r \mid \vh_i) \eqqcolon F(r\mid\vh_i). 
\end{align*}
Therefore, ${\mathbb{E}\left[ S_n(r) \right] = 0}$. Moreover, $\vert Y_i(r) \vert \leq 1$. To prove that $S_n(r) \xrightarrow[n\to\infty]{\mathbb{P}} 0$, we leverage Chebyshev's inequality
\begin{equation*}
    \mathbb{P}(\vert S_n(r) \vert > \epsilon) \leq \frac{\Var(S_n(r))}{\epsilon^2}.
\end{equation*}
and then expand the variance as 
\begin{align}\label{eq:stochastic_variance}
    \Var(S_n(r)) = \Var\left( \sum_{i=1}^n w^{(n)}_iY_i \right) & = \sum_{i=1}^n (w^{(n)}_i)^2 \Var(Y_i) + \sum_{i \neq j } w^{(n)}_iw^{(n)}_j\Cov(Y_i, Y_j)\\
    & \le \sum_{i=1}^n (w^{(n)}_i)^2 \Var(Y_i) + \sum_{i \neq j } w^{(n)}_iw^{(n)}_j\vert \Cov(Y_i, Y_j)\vert. \nonumber
\end{align}
Starting by the variances term, since $\vert Y_i \vert \leq 1$, we have $\Var(Y_i) \leq 1$, thus the first term of \autoref{eq:stochastic_variance} is bounded $\sum_{i=1}^n (w^{(n)}_i)^2 \Var(Y_i) \leq \sum_{i=1}^n (w^{(n)}_i)^2 \to 0$ thanks to Assumption~\ref{ass:weights}~(i).
We now move to the covariances. From Lemma 3 of \cite{doukhan1994mixing}, we know that since the process generating $(\vx_i, r_{i+H})$ is mixing by Assumption~\ref{ass:mixing}, for bounded variables $Y_i, Y_j$ that are function of this process the following holds
\begin{equation*}
    \left\vert \Cov(Y_i, Y_j) \right\vert \leq 4 \alpha(\vert i-j \vert),
\end{equation*}
where $\alpha(k)$ is the mixing coefficient. Hence, in our case, 
\begin{align*}
     \sum_{i\neq j} w^{(n)}_iw^{(n)}_j\vert \Cov(Y_i, Y_j) \vert &\leq 4  \sum_{i\neq j} w^{(n)}_iw^{(n)}_j \alpha(\vert i-j \vert).
\end{align*}
We can split the sum in two terms 
\begin{equation*}
    4 \underbrace{\sum_{1 \leq \vert i - j \vert \leq K} w^{(n)}_iw^{(n)}_j \alpha(\vert i-j \vert)}_{(*)} + 4 \underbrace{\sum_{\vert i - j \vert > K} w^{(n)}_iw^{(n)}_j \alpha(\vert i-j \vert)}_{(**)} 
\end{equation*}
such that for all lags $k > K$ the mixing coefficient is arbitrary small, i.e., for any $\epsilon > 0, \alpha(k) < \epsilon$~(this is always possible). 
\begin{align*}
    (*)  & \leq \left(\max_{1\leq k \leq K} \alpha(k)\right) \sum_{1 \leq \vert i - j \vert \leq K} w^{(n)}_iw^{(n)}_j   \\
    & \le A_{max} \sum_{1 \le |i-j| \le K} w_i^{(n)} w_j^{(n)}\\
    & = 2 A_{max} \sum_{k=1}^K \sum_{i=1}^{n-k} w_i^{(n)} w_{i+k}^{(n)} 
\end{align*}

Using the inequality $2ab \le a^2 + b^2$:
\begin{align*}
    2 A_{max} \sum_{k=1}^K \sum_{i=1}^{n-k} w_i^{(n)} w_{i+k}^{(n)} 
    & \le A_{max} \sum_{k=1}^K \left(\sum_{i=1}^{n-k} (w_i^{(n)})^2 + \sum_{i=1}^{n-k} (w_{i+k}^{(n)})^2 \right)\\
    & \le 2 A_{max} \sum_{k=1}^K \sum_{i=1}^{n} (w_i^{(n)})^2\\
    & = 2 A_{max} K \sum_{i=1}^n (w_i^{(n)})^2 \\
    &= \frac{2 A_{max} K}{m_n} \to 0.
\end{align*}

For the second term
\begin{align*}
    (**) &= \sum_{\vert i - j \vert > K} w^{(n)}_iw^{(n)}_j \alpha(\vert i-j \vert)\\
    &\leq \epsilon \sum_{\vert i - j \vert > K} w^{(n)}_iw^{(n)}_j \\
    &\leq \epsilon
\end{align*}
where $\epsilon$ is arbitrarily small.
}

{
\paragraph{Term (II):} 
\begin{equation*}
    \bar{F}_n(r \mid \vh_t) - {F}(r \mid \vh_t) = \sum_{i=1}^n w^{(n)}_i(\vh_t)\big(F(r\mid \vh_i)-F(r\mid \vh_t)\big)
\end{equation*}
Because weights are a convex combination (nonnegative, sum to 1),
\begin{equation*}
    |\bar{F}_n(r \mid\vh_t) - {F}(r \mid\vh_t)| \leq \sum_{i=1}^n w^{(n)}_i(\vh_t)\big|F(r\mid \vh_i)-F(r\mid \vh_t)\big|
\end{equation*}
Now we use Lemma~\ref{lem:regularity}. For any $\epsilon > 0$ we can choose $\delta > 0$ so that $\Vert \vh_i - \vh_t \Vert < \delta$ implies ${\sup_b\big|F(r\mid \vh_i)-F(r\mid \vh_t)\big| < \epsilon}$. 
We can split the sum into indices with $\Vert \vh_i - \vh_t \Vert < \delta$ and ${\Vert \vh_i - \vh_t \Vert \geq \delta}$:
\begin{align*}
    \sum_{i=1}^n w^{(n)}_i(\vh_t)\big|F(r\mid \vh_i)-F(r\mid \vh_t)\big| &= \sum_{i:\Vert \vh_i - \vh_t \Vert < \delta} w^{(n)}_i(\vh_t)\big|F(r\mid \vh_i)-F(r\mid \vh_t)\big| \\
    &+ \sum_{i:\Vert \vh_i - \vh_t \Vert \geq \delta} w^{(n)}_i(\vh_t)\big|F(r\mid \vh_i)-F(r\mid \vh_t)\big|
\end{align*}
The first term is smaller than $\epsilon$,
\begin{equation*}
    \sum_{i:\Vert \vh_i - \vh_t \Vert < \delta} w^{(n)}_i(\vh_t)\big|F(r\mid \vh_i)-F(r\mid \vh_t)\big| < \epsilon \sum_{i:\Vert \vh_i - \vh_t \Vert < \delta} w^{(n)}_i(\vh_t) < \epsilon
\end{equation*}
which we can make arbitrarily small.
The second term goes to 0 in probability under Assumptions~\ref{ass:weights}-\ref{subass:temperature_cooling} and ~\ref{ass:weights}-\ref{subass:vanish_weights}: as the temperature goes to zero, the weight concentrates on observations $\vh_i$ close to $\vh_t$, and the weight of points not in the neighborhood of $\vh_t$ vanishes
\begin{equation*}
    \sum_{i:\Vert \vh_i - \vh_t \Vert \geq \delta} w^{(n)}_i(\vh_t)\big|F(r\mid \vh_i)-F(r\mid \vh_t)\big| \leq 2 \sum_{i:\Vert \vh_i - \vh_t \Vert \geq \delta} w^{(n)}_i(\vh_t) \xrightarrow[n\to \infty]{\mathbb{P}} 0.
\end{equation*}
}

{
\paragraph{Uniform convergence:} 

The proof here is analogous to proofs of the Glivenko–Cantelli theorem. Consider points
\begin{equation*}
-\infty = r_0 < r_1 < ... < r_{K-1} < r_K = \infty
\end{equation*}
such that for all $\epsilon > 0$ 
\begin{equation*}
F(r_k\mid \vh_t) - F(r_{k-1}\mid \vh_t) \leq \epsilon.
\end{equation*}

This implies
\begin{equation*}
\sup_{r\in \sR}\big|\widehat F_n(r\mid \vh_t)-F(r\mid \vh_t)\big| \leq \max_{1\leq k \leq K} \big|\widehat F_n(r_k\mid \vh_t)-F(r_k\mid \vh_t)\big|\ + \epsilon
\end{equation*}
Since we already proven pointwise convergence, we have that
\begin{equation*}
\max_{1\leq k \leq K} \big|\widehat F_n(r_k\mid \vh_t)-F(r_k\mid \vh_t)\big|\ \xrightarrow[n\to\infty]{\mathbb P}\; 0
\end{equation*}
and since $\epsilon$ can be made arbitrary small 
\begin{equation*}
 \sup_{r\in\mathbb R} \,\big|\widehat F_n(r\mid \vh_t)-F(r\mid \vh_t)\big|
    \;\xrightarrow[n\to\infty]{\mathbb P}\; 0.
\end{equation*}
}
\end{proof}

\begin{proof}[Proof of Corollary~\ref{cor:asym_cond_cov}]
Starting from
\begin{equation*}
    \mathbb{P}\left( y_{t+H} \in \hat{C}_{t}^\alpha (\hat{y}_{t+H}) \mid \vh_t \right) = \mathbb{P}\left( r_{t+H} \in \left[ \widehat{Q}_{\alpha/2}(\vh_t), \widehat{Q}_{1-\alpha/2}(\vh_t) \right] \mid \vh_t \right)
\end{equation*}
so we have
\begin{equation*}
    \mathbb{P}\left( r_{t+H} \in \left[\widehat{Q}_{\alpha/2}(\vh_t), \widehat{Q}_{1-\alpha/2}(\vh_t)\right] \mid \vh_t \right) = F\left(\widehat{Q}_{1-\alpha/2}(\vh_t) \big\vert\; \vh_t\right) - F\left(\widehat{Q}_{\alpha/2}(\vh_t) \big\vert\; \vh_t\right)
\end{equation*}
and since Theorem~\ref{thm:consistency} gives us the consistency of the quantile estimator to the true conditional quantile, we have that
\begin{align*}
    &F\left(\widehat{Q}_{1-\alpha/2}(\vh_t) \big\vert\; \vh_t\right) - F\left(\widehat{Q}_{\alpha/2}(\vh_t) \big\vert\; \vh_t\right)  \xrightarrow[n\to\infty]{\mathbb{P}} 1-\alpha.
\end{align*}
\end{proof}

%% file: appendix_datasets.tex
\section{Datasets and point predictors}\label{app:datasets}

\subsection{Datasets}
In the experiments that we performed, we considered four datasets coming from different real-world scenarios and application domains.

\paragraph{Solar data} The solar data comes from the US National Solar Radiation Database~\citep{sengupta2018the}. We worked with the dataset of 50 time series generated by~\cite{auer2023conformal}, each consisting of 52,609 half-hourly measurements. These were resampled to obtain 26,304 hourly observations. Solar radiation (\texttt{dhi}, Diffuse Horizontal Irradiance) was designated as the target variable, and the dataset includes eight additional environmental features:
\begin{itemize}
    \item[-] \texttt{dni}, Direct Normal Irradiance: amount of solar radiation received per unit area from the sun’s direct beam on a surface kept perpendicular (normal) to the sun’s rays.
    \item[-] \texttt{dew\_point}: the temperature at which air becomes saturated with water vapor, causing condensation (dew) to form.
    \item[-] \texttt{air\_temperature}: the ambient temperature of the air.
    \item[-] \texttt{wind\_speed}: the rate at which air is moving horizontally.
    \item[-] \texttt{total\_precipitable\_water}: the depth of water in a column of the atmosphere if all its water vapor were condensed and precipitated.
    \item[-] \texttt{relative\_humidity}: the percentage ratio of the current water vapor in the air to the maximum amount the air can hold at that temperature.
    \item[-] \texttt{solar\_zenith\_angle}: the angle between the sun’s rays and the vertical direction at a given location.
    \item[-] \texttt{surface\_albedo}: the fraction of incoming solar radiation reflected by the Earth’s surface.
\end{itemize}

\paragraph{Air Quality data} The dataset was first introduced in~\cite{zhang2017cautionary} and comprises 35,064 measurements from 12 locations in Beijing, China. It includes two air-quality measurements (PM10 and PM2.5); for our experiments, we used PM10 as the target and excluded PM2.5 from the feature set. The dataset also contains ten additional environmental variables:
\begin{itemize}
    \item[-] \texttt{SO$_2$}: sulfur dioxide.
    \item[-] \texttt{NO$_2$}: nitrogen dioxide.
    \item[-] \texttt{CO}: carbon monoxide.
    \item[-] \texttt{O$_3$} ozone.
    \item[-] \texttt{TEMP}: ambient air temperature.
    \item[-] \texttt{PRES}: atmospheric pressure.
    \item[-] \texttt{DEWP}: dew point temperature.
    \item[-] \texttt{RAIN}: amount of precipitation.
    \item[-] \texttt{wd}: the direction from which the wind is blowing. This was encoded following the method in~\cite{auer2023conformal}.
    \item[-] \texttt{WSPM}: horizontal wind speed.
\end{itemize}

\paragraph{Exchange rate data} The collection of the daily exchange rates of eight foreign countries, including Australia, Britain, Canada, Switzerland, China, Japan, New Zealand, and Singapore. The data span 1990–2016, yielding 7,588 samples.

\paragraph{ACEA data} The data was first introduced in~\cite{bianchi2015shortterm}. The time series represents electricity load provided by ACEA (\textit{Azienda Comunale Energia e Ambiente}), the company that manages the electric and hydraulic distribution in the city of Rome, Italy. It was sampled every 10 minutes and contains 137,376 observations.

\subsection{Point predictors}
For each dataset, we trained three distinct point-forecasting models that span different modelling paradigms. First, a \gls{rnn} built with \gls{gru} cells and a single layer with a hidden state size of 32. Second, a decoder-only Transformer configured with hidden size 32, a feed-forward size of 64, two attention heads, three stacked layers, and a dropout rate of 0.1. Third, an \gls{arima} model of order ($3,1,3$). 
The first two models were trained and configured similarly to~\cite{cini2025relational}, i.e. by minimizing the \gls{mae} loss with the \gls{adam} optimizer for 200 epochs and a batch size of 32.

For each dataset, every model was trained independently using $40\%$ of that dataset’s samples as the training set. The first $25\%$ of the calibration split is used as a validation set for early stopping.

%% file: appendix_hyperparameters.tex
\section{Hyperparameters and experimental setup}\label{app:hyperparameters}

\subsection{Evaluation metrics}
Given a prediction interval $\widehat{C}^\alpha(\hat{y}_{t}) = \left[\hat{y}_{t} + \hat{q}^{\alpha/2}_{t}, \hat{y}_{t} + \hat{q}^{1-\alpha/2}_{t}\right]$ computed for a desired confidence level $\alpha$ and true observation $y_t$, we evaluated our method against the baselines using three fundamental metrics in uncertainty quantification.

\paragraph{Coverage gap} The coverage gap is defined as
\begin{equation*}
    \Delta\text{Cov} = 100 \;\left(\mathds{1}(y_t\in \widehat{C}^\alpha(\hat{y}_{t})) - (1-\alpha)\right).
\end{equation*}
and indicates the difference between the target coverage $1-\alpha$ and the observed one. 

\paragraph{Prediction interval width} This metric quantifies the width of the prediction intervals, i.e., how uncertain the model is about prediction $\hat{y}_t$
\begin{equation*}
    \text{PI-Width} = \hat{q}^{1-\alpha/2}_{t} - \hat{q}^{\alpha/2}_{t}.
\end{equation*}

\paragraph{Winkler score} This is a combined metric that adds to the $\text{PI-Width}$ a penalty if the actual value $y_t$ falls outside the \gls{pi}. The penalty is proportional by a factor $\frac{2}{\alpha}$ to the degree of error of the \gls{pi}, i.e. how far the actual value is from the closest bound of the interval.
The Winkler score is defined as:
\begin{equation*}
    W = \begin{cases}
        (\hat{q}^{1-\alpha/2}_t - \hat{q}^{\alpha/2}_t) + \frac{2}{\alpha}(\hat{q}^{\alpha/2}_t - y_t) & \text{if } y_t < \hat{q}^{\alpha/2}_t, \\
        (\hat{q}^{1-\alpha/2}_t - \hat{q}^{\alpha/2}_t) & \text{if } \hat{q}^{\alpha/2}_t \leq y_t \leq \hat{q}^{1-\alpha/2}_t, \\
        (\hat{q}^{1-\alpha/2}_t - \hat{q}^{\alpha/2}_t) + \frac{2}{\alpha}(y_t - \hat{q}^{1-\alpha/2}_t) & \text{if } y_t > \hat{q}^{1-\alpha/2}_t.
    \end{cases}
\end{equation*}

\subsection{\gls{esn} Hyperparameters}
If properly configured, \glspl{esn} can provide rich and stable representations of the temporal dependencies in the data. The main hyperparameters that we considered are: the \textbf{reservoir size}, which represents the dimension $D_h$ of the high-dimensional space where reservoir's states evolve (the ‘‘capacity'' of the reservoir); the \textbf{spectral radius} $\rho(\mW_h)$ of the state update weight matrix $\mW_h$, which governs the stability of the internal dynamics of the reservoir; the \textbf{leak rate}, which controls the proportion of information from the previous reservoir state that is preserved and added to the new state after the non-linearity; and the \textbf{input scaling}, which regulates the magnitude of the input signal before it is projected into the reservoir, and thereby control the degree of non-linearity that will be applied to the state update.

Other hyperparameters that define how the \gls{esn} behaves are the connectivity of the reservoirs, which denotes the proportion of non-zero connections in the state update weight matrix and thereby determines the sparsity structure of the reservoir, the amount and magnitude of noise injected during state update, and the activation function for the non-linearity. Sometimes, also the topology of the reservoir can lead to meaningful dynamical properties (a ring-like arrangement of the reservoir's internal connection as opposed to random connections).

\subsection{{\gls{cornn} architecture}}
\label{app:cornn_arch}
The architecture of \gls{cornn} is identical to the model called \textsc{CoRNN} introduced by \cite{cini2025relational}. It consists in a encoder-decoder architecture, in which the encoder is implemented as multiple \gls{gru} layers, and the decoder is a \gls{mlp} that maps the learned representations to the predictions of the quantiles. The model is trained end-to-end by minimizing the pinball loss.

\subsection{Model selection}
We performed hyperparameter tuning for all combinations of point predictors and datasets on a validation set, depending on the model.

\paragraph{\gls{rescp}} For our method, we used $10\%$ of the calibration set for validation. To perform model selection, we run a grid search over the spectral radius for values in the interval $[0.5, 1.5]$, for the leak rate $[0.5, 1]$ ($1$ meaning no leak), for the input scaling $[0.1, 3]$, for the temperature $[0.01, 2]$ and for the size of the calibration window $[100, \text{‘‘all''}]$, where ‘‘$\text{all}$'' considers all states in the calibration set. 
The remaining \gls{esn} hyperparameters where kept fixed (reservoir size $512$, connectivity $0.2$, $\tanh$ activation function, no injected noise during state update, random connections).
Given the high scalability of the method, performing such a grid search was very fast.
The same was done for \textbf{\gls{rescqr}}, but only on the \gls{esn} hyperparameters (reservoir's size, spectral radius, leak rate and input scaling). In this setting, to train the readout we used the \gls{adam} optimizer~\citep{kingma2015adam} with an initial learning rate of 0.003, which was reduced by $50\%$ every time the loss did not decrease for 10 epochs in a row. Each epoch consisted of mini-batches of 64 samples per batch.

\paragraph{\gls{nexcp}} The search was made for the parameter $\rho$ used to control the decay rate for the weights. The search was over the values $\{ 0.999, 0.99, 0.95, 0.9 \}$.

\paragraph{\gls{scpi}} Since \gls{scpi}~\citep{chen2023sequential} required fitting a quantile random forest at each time step, the computational demand to use the official implementation was limiting. Therefore, we used the implementation provided by~\cite{auer2023conformal}, and followed the same protocol for the training. \gls{scpi} was run using a fixed window of $100$ time steps, which corresponds to the longest window considered in~\cite{chen2023sequential}.

\paragraph{\gls{hopcpt}} We followed the same model selection procedure and hyperparameter search of the original paper~\citep{auer2023conformal}. The model was trained for 3,000 epochs in each experiment, validating every 5 epochs on $50\%$ of the calibration data. AdamW~\citep{loshchilov2018decoupled} with standard parameters ($\beta_1 = 0.9$, $\beta_2 = 0.999$, $\delta=0.01$) was used to optimize the model. A batch size of 4 time series was used. The model with smallest PI-Width and non-negative $\Delta$Cov was selected. If no models achieved non-negative $\Delta$Cov, then the one with the highest $\Delta$Cov was chosen.

\paragraph{\gls{cornn}} Also for this baseline, we followed the same model selection procedure as in the original paper~\citep{cini2025relational}. The number of neurons and the number of \gls{gru} layers was tuned with a small grid search on $10\%$ of the calibration data.

\subsection{Resulting hyperparameters}

\input{hyperparameters_table}

The best hyperparameters found with the grid search described in the previous section can be found in~\autoref{tab:hyperparameters}.

%% file: hyperparameters_table.tex
\begin{table*}[t]
\centering
\caption{Best hyperparameter values found with the grid search described in the previous section for each dataset and base model combination.}
\vspace{5pt}
\setlength{\tabcolsep}{2.5pt}
\setlength{\aboverulesep}{0pt}
\setlength{\belowrulesep}{0pt}
\renewcommand{\arraystretch}{1.05}
\resizebox{\textwidth}{!}{

\begin{tabular}{l|l!{\vrule width 1pt}c|c!{\vrule width 1pt}c|c!{\vrule width 1pt}c|c!{\vrule width 1pt}c!{\vrule width 1pt}c}
\multicolumn{2}{c!{\vrule width 1pt}}{}  & \multicolumn{2}{c!{\vrule width 1pt}}{\textbf{Spectral Radius}} & \multicolumn{2}{c!{\vrule width 1pt}}{\textbf{Leak Rate}} & \multicolumn{2}{c!{\vrule width 1pt}}{\textbf{Input Scaling}} & \textbf{Temperature} & \textbf{Sliding Window} \\
    & Model        & \gls{rescp} & \gls{rescqr} & \gls{rescp} & \gls{rescqr} & \gls{rescp} & \gls{rescqr} & \gls{rescp} & \gls{rescp} \\
\midrule[1.5pt]
\multirow{3}{*}{\rotatebox{90}{Solar}}
 &  RNN   & 0.9 & 0.9 & 0.75 & 1.0 & 0.7 & 0.1 & 0.1 & 3900 \\
\cmidrule{2-10}
 & Transformer   & 0.9 & 0.9 & 0.8 & 1.0 & 0.4 & 0.1 & 0.1 & 7600 \\
\cmidrule{2-10}
 & ARIMA   & 1.0 & 0.95 & 0.8 & 0.9 & 0.2 & 0.1 & 0.1 & 1500 \\
\midrule[1.5pt]
\multirow{3}{*}{\rotatebox{90}{Beijing}}
 & RNN   & 1.3 & 0.99 & 0.95 & 1.0 & 0.25 & 0.1 & 0.1 & 3200 \\
\cmidrule{2-10}
 & Transformer   & 1.0 & 0.9 & 0.8 & 1.0 & 0.25 & 0.1 & 0.1 & 10500 \\
\cmidrule{2-10}
 & ARIMA   & 1.45 & 0.99 & 0.65 & 1.0 & 0.75 & 0.1 & 0.15 & 3800 \\
\midrule[1.5pt]
\multirow{3}{*}{\rotatebox{90}{Exch.}}
 & RNN   & 0.95 & 0.9 & 0.8 & 0.9 & 0.5 & 0.25 & 0.1 & 1000 \\
\cmidrule{2-10}
 & Transformer   & 0.99 & 0.99 & 0.8 & 0.9 & 0.5 & 0.1 & 0.1 & 1000 \\
\cmidrule{2-10}
 & ARIMA   & 0.95 & 0.99 & 0.8 & 0.8 & 0.5 & 0.1 & 0.1 & 1000 \\
\midrule[1.5pt]
\multirow{3}{*}{\rotatebox{90}{ACEA}}
 & RNN   & 0.95 & 0.99 & 0.8 & 0.8 & 0.25 & 0.25 & 0.1 & 1000 \\
\cmidrule{2-10}
 & Transformer   & 0.99 & 0.99 & 0.8 & 0.8 & 0.25 & 0.25 & 0.1 & 3000 \\
\cmidrule{2-10}
 & ARIMA   & 0.95 & 0.9 & 0.8 & 1.0 & 0.25 & 0.1 & 0.1 & 1000 \\
\bottomrule[1.5pt]
\end{tabular}
}

\label{tab:hyperparameters}
\end{table*}

%% file: appendix_additional_alpha.tex
\section{{Additional results}}\label{app:calib_curves}
\input{table_alpha_0.05}
\subsection{{Results with different significance levels}}
{
In addition to the results presented in the main text for miscoverage level $\alpha=0.1$, we also evaluate the performance of \gls{rescp} and \gls{rescqr} at significance levels $\alpha=0.05$ and $\alpha=0.15$.
Tables \ref{tab:results_alpha_0.05} and \ref{tab:results_alpha_0.15} summarize the results for these additional significance levels across all datasets and base forecasters and against all the baselines considered in the main text.
The results are consistent with those discussed in the main text for $\alpha=0.1$, confirming the effectiveness of \gls{rescp} in producing well-calibrated and accurate \glspl{pi} across different significance levels, datasets, and base point predictors.
\input{table_alpha_0.15}
}
\subsection{{Calibration curves}}
{
    To better evaluate the calibration performance of \gls{rescp} and \gls{rescqr}, we plot their calibration curves for all datasets and base models.
    The results are shown in \autoref{fig:calib_grid}. The calibration curve plots the target coverage against the observed coverage, where a perfectly calibrated method would follow the diagonal line $y=x$, which is depicted as a black dashed line in the plots.
    The target coverages $\{ 0.7, 0.75, 0.8, 0.85, 0.9, 0.95 \}$, are shown on the $x$-axis, and the corresponding observed coverages are shown on the $y$-axis.
    Each row of the grid corresponds to a dataset, while each column to a base point forecaster. 
    The results confirm the findings discussed in the main text for $\alpha=0.1$, showing that \gls{rescp} produces well-calibrated \glspl{pi} across different significance levels, datasets, and base models.
    \gls{rescqr} also generally generates accurate prediction intervals, but its performance are poor in some benchmarks such as ACEA.
    \input{imgs/combined_rescp_rescqr_calib_plot}
}
\subsection{{Synthetic dataset}}
{
    To further validate the effectiveness of \gls{rescp}, we conduct an ablation study on synthetic data where the ground truth is known. We generate $10000$ samples from a synthetic \gls{ar} process of order 1 
    \begin{equation*}
        y_t = \phi_ty_{t-1} + \epsilon_t \quad \text{with} \quad \epsilon_t \sim \mathcal{N}(0,1)
    \end{equation*}
    with time-varying coefficients $\phi_t$ across the calibration and test sets. Specifically, we introduce three change points in the coefficient sequence $\phi$ between the calibration and test sets. Overall, $\phi_t$ assumes the values $\{ -0.9, 0.3, -0.5, 0.7\}$. The first change point occurs at the midpoint of the calibration set, while the remaining two are placed at equal intervals over the rest of the series.
    The same setup as in \autoref{sec:exp_setup} is used: data is split into training, calibration, and test sets with proportions $40\%/40\%/20\%$, and an \gls{rnn} is employed as the base forecaster.
    We then compare the performance of \gls{rescp} by removing its two key components: (i) the time-dependent weights, (ii) the sliding window of past residuals used for calibration, and (iii) the combination of the two. The results, reported in \autoref{tab:synthetic_ablation}, demonstrate that both components significantly contribute to the performance of \gls{rescp}, with the full model achieving the best results in terms of coverage. This ablation study on synthetic data further confirms the robustness and effectiveness of \gls{rescp} in capturing the underlying dynamics of time series data.
    \input{synthetic_ablation_table}
}
\subsection{{Visualization of the adaptive \glspl{pi}}}
{
    \input{imgs/adaptive_pis_elec.tex}
    To give a better intuition of how \gls{rescp} adapts the width of the \glspl{pi} over time in response to changes in the data distribution, we visualize some examples of adaptive \glspl{pi} produced by \gls{rescp} on the ACEA and Solar datasets in \autoref{fig:elec_adaptive_pis} and \autoref{fig:solar_adaptive_pis}, respectively.
    These plots illustrate how \gls{rescp} effectively adjusts the uncertainty quantification based on the observed patterns.
    \input{imgs/adaptive_pis_solar.tex}
}

%% file: table_alpha_0.05.tex
\begin{table*}[t]
\centering
\small
\caption{{Performance comparison for $\alpha=0.05$. 
$\Delta$Cov values are color-coded for undercoverage cases: \textcolor{tabolive}{yellow} (1-2\%), \textcolor{taborange}{orange} (2-4\%), \textcolor{tabred}{red} ($>$4\%). 
The top three Winkler scores for each scenario are highlighted: \textbf{\underline{bold+underlined}} (1st), \underline{underlined} (2nd), \dotuline{dot-underlined} (3rd).}}
\setlength{\tabcolsep}{2.5pt}
\setlength{\aboverulesep}{0pt}
\setlength{\belowrulesep}{0pt}
\renewcommand{\arraystretch}{1.05}
\resizebox{\textwidth}{!}{

\begin{tabular}{@{}l|l|l|cccc|ccc@{}}
\multicolumn{3}{c|}{}  & \multicolumn{4}{c|}{\textit{Learning}} & \multicolumn{3}{c}{\textit{Non-learning}} \\
\multicolumn{2}{c}{}    & Metric        & \gls{scpi} & \gls{hopcpt} & \gls{cornn} & \textbf{\gls{rescqr}} & \gls{scp} & \gls{nexcp} & \textbf{\gls{rescp}} \\
\midrule[1.5pt]
\multirow{9}{*}{\rotatebox{90}{Solar}}
 & \multirow{3}{*}{\rotatebox{90}{RNN}} & $\Delta$Cov   & -0.19{\tiny$\pm$0.26} & -0.90{\tiny$\pm$0.22} & -0.33{\tiny$\pm$1.48} & -0.38{\tiny$\pm$0.39} & 0.11 & 0.97 & 0.18{\tiny$\pm$0.22} \\
 &                             & PI-Width      & 124.25{\tiny$\pm$4.19} & 82.29{\tiny$\pm$0.63} & 71.14{\tiny$\pm$1.02} & 77.76{\tiny$\pm$0.87} & 141.23 & 154.71 & 86.56{\tiny$\pm$0.78} \\
 &                             & Winkler       & 214.33{\tiny$\pm$2.03} & 149.88{\tiny$\pm$1.15} & \textbf{\underline{94.73{\tiny$\pm$0.46}}} & \underline{108.42{\tiny$\pm$0.34}} & 239.24 & 225.85 & \dotuline{138.56{\tiny$\pm$1.27}} \\
\cmidrule{2-10}
 & \multirow{3}{*}{\rotatebox{90}{Transf}} & $\Delta$Cov   & -0.40{\tiny$\pm$0.26} & 0.50{\tiny$\pm$0.64} & 0.48{\tiny$\pm$0.78} & 1.05{\tiny$\pm$0.12} & 0.06 & 1.00 & \textcolor{tabolive}{-1.29{\tiny$\pm$0.61}} \\
 &                             & PI-Width      & 123.48{\tiny$\pm$5.82} & 85.55{\tiny$\pm$4.56} & 70.54{\tiny$\pm$1.49} & 76.15{\tiny$\pm$0.46} & 139.52 & 153.24 & 62.61{\tiny$\pm$1.31} \\
 &                             & Winkler       & 215.15{\tiny$\pm$3.01} & 133.84{\tiny$\pm$7.86} & \textbf{\underline{93.31{\tiny$\pm$0.16}}} & \dotuline{107.31{\tiny$\pm$0.35}} & 236.43 & 223.17 & \underline{105.80{\tiny$\pm$2.16}} \\
\cmidrule{2-10}
 & \multirow{3}{*}{\rotatebox{90}{ARIMA}} & $\Delta$Cov   & -0.52{\tiny$\pm$0.17} & 0.75{\tiny$\pm$1.04} & -0.40{\tiny$\pm$0.37} & -0.62{\tiny$\pm$0.15} & 0.07 & 1.21 & 0.21{\tiny$\pm$0.51} \\
 &                             & PI-Width      & 124.61{\tiny$\pm$1.37} & 188.97{\tiny$\pm$6.52} & 70.67{\tiny$\pm$0.55} & 87.34{\tiny$\pm$1.33} & 188.57 & 202.58 & 99.68{\tiny$\pm$1.20} \\
 &                             & Winkler       & 193.31{\tiny$\pm$0.59} & 255.20{\tiny$\pm$16.79} & \textbf{\underline{88.30{\tiny$\pm$0.15}}} & \underline{115.66{\tiny$\pm$0.42}} & 281.61 & 265.15 & \dotuline{133.94{\tiny$\pm$3.25}} \\
\midrule
\multirow{9}{*}{\rotatebox{90}{Beijing}}
 & \multirow{3}{*}{\rotatebox{90}{RNN}} & $\Delta$Cov   & \textcolor{tabolive}{-1.38{\tiny$\pm$1.13}} & \textcolor{tabolive}{-1.06{\tiny$\pm$1.40}} & -0.36{\tiny$\pm$0.55} & -0.58{\tiny$\pm$0.38} & -0.18 & -0.01 & -0.60{\tiny$\pm$0.17} \\
 &                             & PI-Width      & 97.22{\tiny$\pm$6.98} & 112.10{\tiny$\pm$31.41} & 84.51{\tiny$\pm$2.90} & 87.50{\tiny$\pm$3.98} & 98.51 & 100.48 & 86.14{\tiny$\pm$0.95} \\
 &                             & Winkler       & 167.81{\tiny$\pm$2.37} & 183.15{\tiny$\pm$20.62} & \underline{130.62{\tiny$\pm$0.69}} & \dotuline{136.65{\tiny$\pm$1.07}} & 170.36 & 166.04 & \textbf{\underline{124.24{\tiny$\pm$0.11}}} \\
\cmidrule{2-10}
 & \multirow{3}{*}{\rotatebox{90}{Transf}} & $\Delta$Cov   & \textcolor{tabolive}{-1.27{\tiny$\pm$0.46}} & -0.46{\tiny$\pm$0.56} & -0.43{\tiny$\pm$0.63} & \textcolor{tabolive}{-1.05{\tiny$\pm$0.83}} & -0.27 & -0.03 & -0.52{\tiny$\pm$0.32} \\
 &                             & PI-Width      & 95.82{\tiny$\pm$5.28} & 100.94{\tiny$\pm$6.20} & 84.24{\tiny$\pm$3.84} & 87.66{\tiny$\pm$6.01} & 97.75 & 100.17 & 82.11{\tiny$\pm$1.88} \\
 &                             & Winkler       & 167.34{\tiny$\pm$2.11} & 167.90{\tiny$\pm$2.74} & \underline{131.27{\tiny$\pm$1.09}} & \dotuline{137.80{\tiny$\pm$2.26}} & 171.09 & 166.74 & \textbf{\underline{119.46{\tiny$\pm$0.35}}} \\
\cmidrule{2-10}
 & \multirow{3}{*}{\rotatebox{90}{ARIMA}} & $\Delta$Cov   & -0.53{\tiny$\pm$0.36} & \textcolor{tabolive}{-1.20{\tiny$\pm$0.45}} & \textcolor{tabolive}{-1.30{\tiny$\pm$0.36}} & \textcolor{tabolive}{-1.13{\tiny$\pm$0.49}} & -0.23 & -0.14 & -0.29{\tiny$\pm$0.22} \\
 &                             & PI-Width      & 104.81{\tiny$\pm$3.69} & 94.47{\tiny$\pm$2.13} & 81.64{\tiny$\pm$3.94} & 85.17{\tiny$\pm$1.77} & 106.81 & 107.73 & 94.54{\tiny$\pm$1.67} \\
 &                             & Winkler       & 175.14{\tiny$\pm$2.04} & 168.88{\tiny$\pm$1.21} & \textbf{\underline{128.63{\tiny$\pm$1.34}}} & \underline{136.63{\tiny$\pm$1.26}} & 180.18 & 174.82 & \dotuline{154.79{\tiny$\pm$0.33}} \\
\midrule
\multirow{9}{*}{\rotatebox{90}{Exchange}}
 & \multirow{3}{*}{\rotatebox{90}{RNN}} & $\Delta$Cov   & 0.58{\tiny$\pm$0.31} & 1.72{\tiny$\pm$0.00} & \textcolor{tabolive}{-2.00{\tiny$\pm$2.22}} & -0.52{\tiny$\pm$3.06} & 1.35 & 1.18 & 0.54{\tiny$\pm$0.20} \\
 &                             & PI-Width      & 0.0290{\tiny$\pm$0.0008} & 0.0513{\tiny$\pm$0.0000} & 0.0404{\tiny$\pm$0.0011} & 0.0479{\tiny$\pm$0.0019} & 0.0549 & 0.0497 & 0.0260{\tiny$\pm$0.0002} \\
 &                             & Winkler       & \underline{0.0356{\tiny$\pm$0.0005}} & 0.0608{\tiny$\pm$0.0000} & \dotuline{0.0568{\tiny$\pm$0.0002}} & 0.0578{\tiny$\pm$0.0009} & 0.0645 & 0.0604 & \textbf{\underline{0.0321{\tiny$\pm$0.0004}}} \\
\cmidrule{2-10}
 & \multirow{3}{*}{\rotatebox{90}{Transf}} & $\Delta$Cov   & 1.98{\tiny$\pm$0.38} & 1.81{\tiny$\pm$0.00} & 0.94{\tiny$\pm$1.00} & -0.01{\tiny$\pm$2.96} & 2.67 & 1.93 & 0.67{\tiny$\pm$0.11} \\
 &                             & PI-Width      & 0.0310{\tiny$\pm$0.0008} & 0.0501{\tiny$\pm$0.0000} & 0.0441{\tiny$\pm$0.0015} & 0.0463{\tiny$\pm$0.0044} & 0.0646 & 0.0601 & 0.0281{\tiny$\pm$0.0001} \\
 &                             & Winkler       & \underline{0.0375{\tiny$\pm$0.0006}} & 0.0600{\tiny$\pm$0.0000} & \dotuline{0.0585{\tiny$\pm$0.0006}} & 0.0595{\tiny$\pm$0.0011} & 0.0745 & 0.0712 & \textbf{\underline{0.0355{\tiny$\pm$0.0002}}} \\
\cmidrule{2-10}
 & \multirow{3}{*}{\rotatebox{90}{ARIMA}} & $\Delta$Cov   & 1.01{\tiny$\pm$0.45} & 1.39{\tiny$\pm$0.00} & \textcolor{tabolive}{-1.50{\tiny$\pm$0.75}} & \textcolor{tabolive}{-1.76{\tiny$\pm$0.63}} & 1.87 & 1.32 & -0.03{\tiny$\pm$0.15} \\
 &                             & PI-Width      & 0.0293{\tiny$\pm$0.0005} & 0.0482{\tiny$\pm$0.0000} & 0.0404{\tiny$\pm$0.0006} & 0.0419{\tiny$\pm$0.0016} & 0.0489 & 0.0442 & 0.0258{\tiny$\pm$0.0000} \\
 &                             & Winkler       & \underline{0.0357{\tiny$\pm$0.0004}} & 0.0576{\tiny$\pm$0.0000} & 0.0560{\tiny$\pm$0.0004} & \dotuline{0.0547{\tiny$\pm$0.0007}} & 0.0585 & 0.0552 & \textbf{\underline{0.0328{\tiny$\pm$0.0001}}} \\
\midrule
\multirow{9}{*}{\rotatebox{90}{ACEA}}
 & \multirow{3}{*}{\rotatebox{90}{RNN}} & $\Delta$Cov   & \textcolor{tabolive}{-1.20{\tiny$\pm$0.96}} & \textcolor{tabolive}{-1.92{\tiny$\pm$0.00}} & \textcolor{tabred}{-8.29{\tiny$\pm$5.64}} & \textcolor{tabred}{-15.32{\tiny$\pm$4.57}} & -0.89 & -0.29 & 0.32{\tiny$\pm$0.22} \\
 &                             & PI-Width      & 11.13{\tiny$\pm$0.88} & 23.20{\tiny$\pm$0.00} & 19.85{\tiny$\pm$2.59} & 18.18{\tiny$\pm$1.58} & 24.31 & 24.48 & 11.90{\tiny$\pm$0.30} \\
 &                             & Winkler       & \underline{18.38{\tiny$\pm$0.40}} & 32.48{\tiny$\pm$0.00} & 39.72{\tiny$\pm$4.66} & 46.11{\tiny$\pm$6.34} & 32.57 & \dotuline{31.23} & \textbf{\underline{16.31{\tiny$\pm$0.23}}} \\
\cmidrule{2-10}
 & \multirow{3}{*}{\rotatebox{90}{Transf}} & $\Delta$Cov   & \textcolor{tabolive}{-1.75{\tiny$\pm$1.19}} & \textcolor{taborange}{-2.04{\tiny$\pm$0.00}} & \textcolor{tabred}{-10.72{\tiny$\pm$7.28}} & \textcolor{tabred}{-21.13{\tiny$\pm$7.72}} & \textcolor{taborange}{-3.49} & -0.27 & 1.72{\tiny$\pm$0.06} \\
 &                             & PI-Width      & 11.22{\tiny$\pm$0.48} & 22.80{\tiny$\pm$0.00} & 18.44{\tiny$\pm$2.29} & 16.25{\tiny$\pm$2.03} & 21.99 & 25.34 & 12.62{\tiny$\pm$0.37} \\
 &                             & Winkler       & \underline{18.84{\tiny$\pm$0.67}} & 32.54{\tiny$\pm$0.00} & 41.67{\tiny$\pm$9.46} & 52.93{\tiny$\pm$8.53} & 34.43 & \dotuline{32.01} & \textbf{\underline{16.19{\tiny$\pm$0.40}}} \\
\cmidrule{2-10}
 & \multirow{3}{*}{\rotatebox{90}{ARIMA}} & $\Delta$Cov   & 0.92{\tiny$\pm$0.37} & \textcolor{tabolive}{-1.87{\tiny$\pm$0.00}} & \textcolor{tabred}{-17.36{\tiny$\pm$11.98}} & \textcolor{tabred}{-23.84{\tiny$\pm$7.89}} & -0.13 & -0.24 & \textcolor{taborange}{-2.70{\tiny$\pm$0.76}} \\
 &                             & PI-Width      & 14.86{\tiny$\pm$0.40} & 39.74{\tiny$\pm$0.00} & 25.26{\tiny$\pm$4.44} & 20.67{\tiny$\pm$2.38} & 43.65 & 41.42 & 11.68{\tiny$\pm$0.76} \\
 &                             & Winkler       & \underline{21.83{\tiny$\pm$0.19}} & 49.57{\tiny$\pm$0.00} & 63.53{\tiny$\pm$17.63} & 67.04{\tiny$\pm$12.20} & 50.52 & \dotuline{48.69} & \textbf{\underline{18.01{\tiny$\pm$0.53}}} \\
\bottomrule[1.5pt]
\end{tabular}
}

\label{tab:results_alpha_0.05}
\end{table*}

%% file: table_alpha_0.15.tex
\begin{table*}[t]
\centering
\small
\caption{{Performance comparison for $\alpha=0.15$. 
$\Delta$Cov values are color-coded for undercoverage cases: \textcolor{tabolive}{yellow} (1-2\%), \textcolor{taborange}{orange} (2-4\%), \textcolor{tabred}{red} ($>$4\%). 
The top three Winkler scores for each scenario are highlighted: \textbf{\underline{bold+underlined}} (1st), \underline{underlined} (2nd), \dotuline{dot-underlined} (3rd).}}
\setlength{\tabcolsep}{2.5pt}
\setlength{\aboverulesep}{0pt}
\setlength{\belowrulesep}{0pt}
\renewcommand{\arraystretch}{1.05}
\resizebox{\textwidth}{!}{

\begin{tabular}{@{}l|l|l|cccc|ccc@{}}
\multicolumn{3}{c|}{}  & \multicolumn{4}{c|}{\textit{Learning}} & \multicolumn{3}{c}{\textit{Non-learning}} \\
\multicolumn{2}{c}{}    & Metric        & \gls{scpi} & \gls{hopcpt} & \gls{cornn} & \textbf{\gls{rescqr}} & \gls{scp} & \gls{nexcp} & \textbf{\gls{rescp}} \\
\midrule[1.5pt]
\multirow{9}{*}{\rotatebox{90}{Solar}}
 & \multirow{3}{*}{\rotatebox{90}{RNN}} & $\Delta$Cov   & 0.72{\tiny$\pm$0.23} & \textcolor{tabolive}{-1.77{\tiny$\pm$0.33}} & \textcolor{tabolive}{-1.10{\tiny$\pm$1.83}} & 0.10{\tiny$\pm$1.03} & 0.64 & 1.82 & -0.47{\tiny$\pm$0.32} \\
 &                             & PI-Width      & 46.78{\tiny$\pm$0.81} & 47.68{\tiny$\pm$0.46} & 46.84{\tiny$\pm$0.46} & 47.43{\tiny$\pm$0.56} & 47.85 & 55.73 & 46.37{\tiny$\pm$0.36} \\
 &                             & Winkler       & 119.10{\tiny$\pm$0.51} & 96.70{\tiny$\pm$0.32} & \textbf{\underline{69.15{\tiny$\pm$0.14}}} & \underline{75.01{\tiny$\pm$0.11}} & 134.09 & 130.47 & \dotuline{87.44{\tiny$\pm$0.22}} \\
\cmidrule{2-10}
 & \multirow{3}{*}{\rotatebox{90}{Transf}} & $\Delta$Cov   & 0.34{\tiny$\pm$0.31} & 2.88{\tiny$\pm$0.72} & 2.54{\tiny$\pm$4.23} & \textcolor{taborange}{-3.19{\tiny$\pm$16.62}} & 0.60 & 1.81 & 0.18{\tiny$\pm$0.66} \\
 &                             & PI-Width      & 46.56{\tiny$\pm$1.69} & 46.78{\tiny$\pm$2.76} & 46.47{\tiny$\pm$1.25} & 47.06{\tiny$\pm$0.36} & 48.72 & 56.72 & 39.37{\tiny$\pm$0.51} \\
 &                             & Winkler       & 119.81{\tiny$\pm$0.62} & 89.44{\tiny$\pm$3.30} & \textbf{\underline{68.48{\tiny$\pm$0.14}}} & \underline{74.44{\tiny$\pm$0.21}} & 133.18 & 129.62 & \dotuline{75.69{\tiny$\pm$0.39}} \\
\cmidrule{2-10}
 & \multirow{3}{*}{\rotatebox{90}{ARIMA}} & $\Delta$Cov   & 0.87{\tiny$\pm$0.43} & 1.98{\tiny$\pm$2.56} & \textcolor{tabolive}{-1.06{\tiny$\pm$0.43}} & -0.43{\tiny$\pm$0.40} & -0.07 & 2.29 & 0.30{\tiny$\pm$0.61} \\
 &                             & PI-Width      & 71.25{\tiny$\pm$1.42} & 111.82{\tiny$\pm$7.83} & 50.18{\tiny$\pm$0.44} & 60.41{\tiny$\pm$1.26} & 93.98 & 102.74 & 62.24{\tiny$\pm$0.94} \\
 &                             & Winkler       & 124.75{\tiny$\pm$0.35} & 171.47{\tiny$\pm$11.32} & \textbf{\underline{68.52{\tiny$\pm$0.09}}} & \underline{87.26{\tiny$\pm$0.24}} & 179.83 & 174.48 & \dotuline{98.54{\tiny$\pm$4.21}} \\
\midrule
\multirow{9}{*}{\rotatebox{90}{Beijing}}
 & \multirow{3}{*}{\rotatebox{90}{RNN}} & $\Delta$Cov   & -0.94{\tiny$\pm$1.14} & 0.56{\tiny$\pm$4.00} & \textcolor{tabolive}{-1.31{\tiny$\pm$1.19}} & -0.26{\tiny$\pm$1.93} & -0.48 & 0.10 & -0.98{\tiny$\pm$1.17} \\
 &                             & PI-Width      & 55.46{\tiny$\pm$2.03} & 67.60{\tiny$\pm$18.87} & 51.21{\tiny$\pm$1.80} & 54.58{\tiny$\pm$3.94} & 52.93 & 55.07 & 52.21{\tiny$\pm$2.06} \\
 &                             & Winkler       & 103.43{\tiny$\pm$0.64} & 112.44{\tiny$\pm$11.23} & \underline{87.49{\tiny$\pm$0.38}} & \dotuline{90.42{\tiny$\pm$0.38}} & 104.56 & 103.16 & \textbf{\underline{86.72{\tiny$\pm$0.12}}} \\
\cmidrule{2-10}
 & \multirow{3}{*}{\rotatebox{90}{Transf}} & $\Delta$Cov   & -0.69{\tiny$\pm$0.57} & -0.45{\tiny$\pm$1.23} & \textcolor{tabolive}{-1.32{\tiny$\pm$1.71}} & \textcolor{tabolive}{-1.68{\tiny$\pm$1.81}} & -0.46 & 0.11 & -0.75{\tiny$\pm$1.35} \\
 &                             & PI-Width      & 54.51{\tiny$\pm$1.07} & 55.14{\tiny$\pm$3.17} & 50.67{\tiny$\pm$2.82} & 52.98{\tiny$\pm$3.10} & 52.42 & 54.64 & 50.31{\tiny$\pm$2.16} \\
 &                             & Winkler       & 102.74{\tiny$\pm$0.45} & 103.53{\tiny$\pm$0.85} & \underline{87.71{\tiny$\pm$0.41}} & \dotuline{90.97{\tiny$\pm$1.12}} & 104.60 & 103.21 & \textbf{\underline{84.47{\tiny$\pm$0.23}}} \\
\cmidrule{2-10}
 & \multirow{3}{*}{\rotatebox{90}{ARIMA}} & $\Delta$Cov   & 0.74{\tiny$\pm$0.30} & \textcolor{taborange}{-3.08{\tiny$\pm$0.97}} & \textcolor{taborange}{-2.04{\tiny$\pm$0.89}} & \textcolor{taborange}{-2.07{\tiny$\pm$0.79}} & -0.25 & -0.09 & -0.32{\tiny$\pm$0.46} \\
 &                             & PI-Width      & 60.45{\tiny$\pm$0.85} & 53.46{\tiny$\pm$1.10} & 50.58{\tiny$\pm$1.85} & 51.64{\tiny$\pm$1.31} & 60.29 & 60.73 & 54.70{\tiny$\pm$1.13} \\
 &                             & Winkler       & 108.16{\tiny$\pm$0.47} & 107.50{\tiny$\pm$0.68} & \textbf{\underline{86.46{\tiny$\pm$0.55}}} & \underline{90.85{\tiny$\pm$0.38}} & 112.79 & 110.66 & \dotuline{99.59{\tiny$\pm$0.30}} \\
\midrule
\multirow{9}{*}{\rotatebox{90}{Exchange}}
 & \multirow{3}{*}{\rotatebox{90}{RNN}} & $\Delta$Cov   & 5.13{\tiny$\pm$0.59} & 3.05{\tiny$\pm$0.00} & \textcolor{taborange}{-3.28{\tiny$\pm$2.93}} & 2.40{\tiny$\pm$1.73} & 3.06 & 2.26 & 1.77{\tiny$\pm$0.29} \\
 &                             & PI-Width      & 0.0207{\tiny$\pm$0.0004} & 0.0342{\tiny$\pm$0.0000} & 0.0284{\tiny$\pm$0.0009} & 0.0323{\tiny$\pm$0.0010} & 0.0380 & 0.0351 & 0.0180{\tiny$\pm$0.0001} \\
 &                             & Winkler       & \underline{0.0247{\tiny$\pm$0.0002}} & 0.0420{\tiny$\pm$0.0000} & 0.0413{\tiny$\pm$0.0005} & \dotuline{0.0410{\tiny$\pm$0.0006}} & 0.0452 & 0.0436 & \textbf{\underline{0.0234{\tiny$\pm$0.0002}}} \\
\cmidrule{2-10}
 & \multirow{3}{*}{\rotatebox{90}{Transf}} & $\Delta$Cov   & 6.43{\tiny$\pm$0.46} & 3.78{\tiny$\pm$0.00} & 0.25{\tiny$\pm$1.86} & \textcolor{taborange}{-2.20{\tiny$\pm$5.75}} & 6.03 & 4.08 & 2.37{\tiny$\pm$0.15} \\
 &                             & PI-Width      & 0.0221{\tiny$\pm$0.0005} & 0.0341{\tiny$\pm$0.0000} & 0.0301{\tiny$\pm$0.0007} & 0.0313{\tiny$\pm$0.0016} & 0.0483 & 0.0455 & 0.0195{\tiny$\pm$0.0000} \\
 &                             & Winkler       & \textbf{\underline{0.0261{\tiny$\pm$0.0004}}} & \dotuline{0.0421{\tiny$\pm$0.0000}} & 0.0425{\tiny$\pm$0.0006} & 0.0422{\tiny$\pm$0.0003} & 0.0558 & 0.0545 & \underline{0.0262{\tiny$\pm$0.0001}} \\
\cmidrule{2-10}
 & \multirow{3}{*}{\rotatebox{90}{ARIMA}} & $\Delta$Cov   & 6.23{\tiny$\pm$0.72} & 2.96{\tiny$\pm$0.00} & \textcolor{taborange}{-2.26{\tiny$\pm$1.49}} & \textcolor{taborange}{-2.04{\tiny$\pm$3.06}} & 4.25 & 2.77 & 1.03{\tiny$\pm$0.36} \\
 &                             & PI-Width      & 0.0213{\tiny$\pm$0.0009} & 0.0322{\tiny$\pm$0.0000} & 0.0279{\tiny$\pm$0.0006} & 0.0294{\tiny$\pm$0.0016} & 0.0330 & 0.0303 & 0.0175{\tiny$\pm$0.0001} \\
 &                             & Winkler       & \underline{0.0253{\tiny$\pm$0.0006}} & 0.0399{\tiny$\pm$0.0000} & 0.0405{\tiny$\pm$0.0005} & 0.0401{\tiny$\pm$0.0003} & 0.0404 & \dotuline{0.0392} & \textbf{\underline{0.0234{\tiny$\pm$0.0001}}} \\
\midrule
\multirow{9}{*}{\rotatebox{90}{ACEA}}
 & \multirow{3}{*}{\rotatebox{90}{RNN}} & $\Delta$Cov   & \textcolor{tabolive}{-1.84{\tiny$\pm$4.08}} & \textcolor{taborange}{-2.43{\tiny$\pm$0.00}} & \textcolor{tabred}{-16.15{\tiny$\pm$9.30}} & \textcolor{tabred}{-20.50{\tiny$\pm$5.38}} & \textcolor{tabolive}{-1.55} & -0.46 & 1.86{\tiny$\pm$0.40} \\
 &                             & PI-Width      & 7.42{\tiny$\pm$0.71} & 15.93{\tiny$\pm$0.00} & 13.09{\tiny$\pm$1.44} & 13.29{\tiny$\pm$1.31} & 16.34 & 17.07 & 8.06{\tiny$\pm$0.13} \\
 &                             & Winkler       & \underline{12.33{\tiny$\pm$0.15}} & 24.51{\tiny$\pm$0.00} & 28.71{\tiny$\pm$3.73} & 30.47{\tiny$\pm$2.40} & 24.66 & \dotuline{24.15} & \textbf{\underline{11.14{\tiny$\pm$0.11}}} \\
\cmidrule{2-10}
 & \multirow{3}{*}{\rotatebox{90}{Transf}} & $\Delta$Cov   & \textcolor{taborange}{-2.58{\tiny$\pm$0.78}} & \textcolor{taborange}{-2.98{\tiny$\pm$0.00}} & \textcolor{tabred}{-17.58{\tiny$\pm$10.86}} & \textcolor{tabred}{-29.97{\tiny$\pm$8.62}} & \textcolor{tabred}{-5.76} & -0.56 & 4.20{\tiny$\pm$0.23} \\
 &                             & PI-Width      & 7.34{\tiny$\pm$0.12} & 15.26{\tiny$\pm$0.00} & 11.98{\tiny$\pm$1.39} & 10.99{\tiny$\pm$1.51} & 13.66 & 16.73 & 8.58{\tiny$\pm$0.11} \\
 &                             & Winkler       & \underline{12.69{\tiny$\pm$0.10}} & \dotuline{24.40{\tiny$\pm$0.00}} & 29.59{\tiny$\pm$5.71} & 33.91{\tiny$\pm$3.65} & 25.63 & 24.52 & \textbf{\underline{11.42{\tiny$\pm$0.20}}} \\
\cmidrule{2-10}
 & \multirow{3}{*}{\rotatebox{90}{ARIMA}} & $\Delta$Cov   & 0.39{\tiny$\pm$1.05} & \textcolor{tabred}{-4.81{\tiny$\pm$0.00}} & \textcolor{tabred}{-36.21{\tiny$\pm$8.89}} & \textcolor{tabred}{-34.84{\tiny$\pm$9.41}} & \textcolor{tabolive}{-1.27} & -0.72 & 1.16{\tiny$\pm$1.10} \\
 &                             & PI-Width      & 10.92{\tiny$\pm$0.29} & 30.89{\tiny$\pm$0.00} & 14.51{\tiny$\pm$1.64} & 13.39{\tiny$\pm$1.91} & 33.68 & 32.72 & 8.31{\tiny$\pm$0.52} \\
 &                             & Winkler       & \underline{15.51{\tiny$\pm$0.12}} & 41.82{\tiny$\pm$0.00} & 47.54{\tiny$\pm$5.69} & 44.95{\tiny$\pm$6.32} & 43.00 & \dotuline{40.68} & \textbf{\underline{11.66{\tiny$\pm$0.51}}} \\
\bottomrule[1.5pt]
\end{tabular}
}
\label{tab:results_alpha_0.15}
\end{table*}

%% file: imgs/combined_rescp_rescqr_calib_plot.tex
\begin{figure}
    \centering
    \includegraphics[width=0.9\linewidth, trim=50 20 50 0]{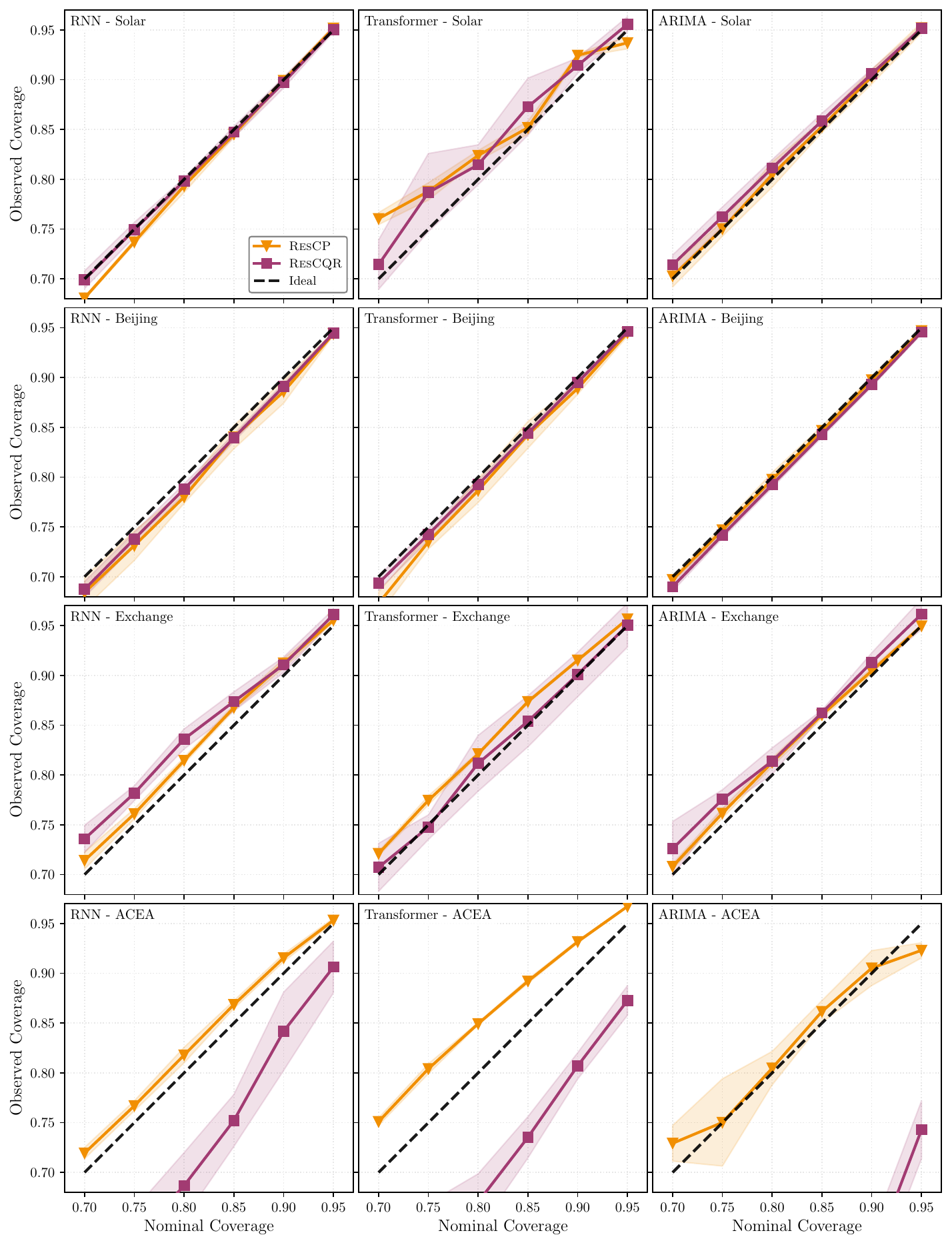}
    \caption{{Calibration curves of \gls{rescp} and \gls{rescqr} over all datasets and baselines.}}
    \label{fig:rescp_calib_grid}
\end{figure}

%% file: synthetic_ablation_table.tex
\begin{table}[t]
\centering
\small
\caption{{Ablation study on the AR synthetic dataset for three miscoverage levels ($\alpha \in \{0.05, 0.1, 0.15\}$).}}
\setlength{\tabcolsep}{4pt}
\setlength{\aboverulesep}{0pt}
\setlength{\belowrulesep}{0pt}
\renewcommand{\arraystretch}{1.2}

\begin{tabular}{@{}c|l|cccc||cc@{}}
$\alpha$ & Metric & \textbf{\gls{rescp}} & {\scriptsize No decay} & {\scriptsize No window} & {\scriptsize No decay, no window} & NExCP & SCP \\
\midrule[1.5pt]
\multirow{2}{*}{\rotatebox{90}{$0.05$}} & Coverage & 94.41{\tiny$\pm$0.21} & 93.50{\tiny$\pm$0.23} & 91.62{\tiny$\pm$0.27} & 88.39{\tiny$\pm$0.17} & 94.34 & 86.11 \\
 & $\Delta$Cov & -0.59{\tiny$\pm$0.21} & \textcolor{tabolive}{-1.50{\tiny$\pm$0.23}} & \textcolor{taborange}{-3.38{\tiny$\pm$0.27}} & \textcolor{tabred}{-6.61{\tiny$\pm$0.17}} & -0.66 & \textcolor{tabred}{-8.89} \\
\cmidrule(lr){1-8}
\multirow{2}{*}{\rotatebox{90}{$0.1$}} & Coverage & 89.75{\tiny$\pm$0.28} & 88.25{\tiny$\pm$0.41} & 84.45{\tiny$\pm$0.19} & 80.85{\tiny$\pm$0.17} & 88.84 & 78.48 \\
 & $\Delta$Cov & -0.25{\tiny$\pm$0.28} & \textcolor{tabolive}{-1.75{\tiny$\pm$0.41}} & \textcolor{tabred}{-5.55{\tiny$\pm$0.19}} & \textcolor{tabred}{-9.15{\tiny$\pm$0.17}} & \textcolor{tabolive}{-1.16} & \textcolor{tabred}{-11.52} \\
\cmidrule(lr){1-8}
\multirow{2}{*}{\rotatebox{90}{$0.15$}} & Coverage & 84.93{\tiny$\pm$0.34} & 83.11{\tiny$\pm$0.35} & 78.33{\tiny$\pm$0.20} & 73.99{\tiny$\pm$0.28} & 84.14 & 71.92 \\
 & $\Delta$Cov & -0.07{\tiny$\pm$0.34} & \textcolor{tabolive}{-1.89{\tiny$\pm$0.35}} & \textcolor{tabred}{-6.67{\tiny$\pm$0.20}} & \textcolor{tabred}{-11.01{\tiny$\pm$0.28}} & -0.86 & \textcolor{tabred}{-13.08} \\
\bottomrule
\end{tabular}
\label{tab:synthetic_ablation}
\end{table}

%% file: imgs/adaptive_pis_elec.tex
\begin{figure}
    \centering
    \includegraphics[width=0.9\linewidth, trim=50 20 50 0]{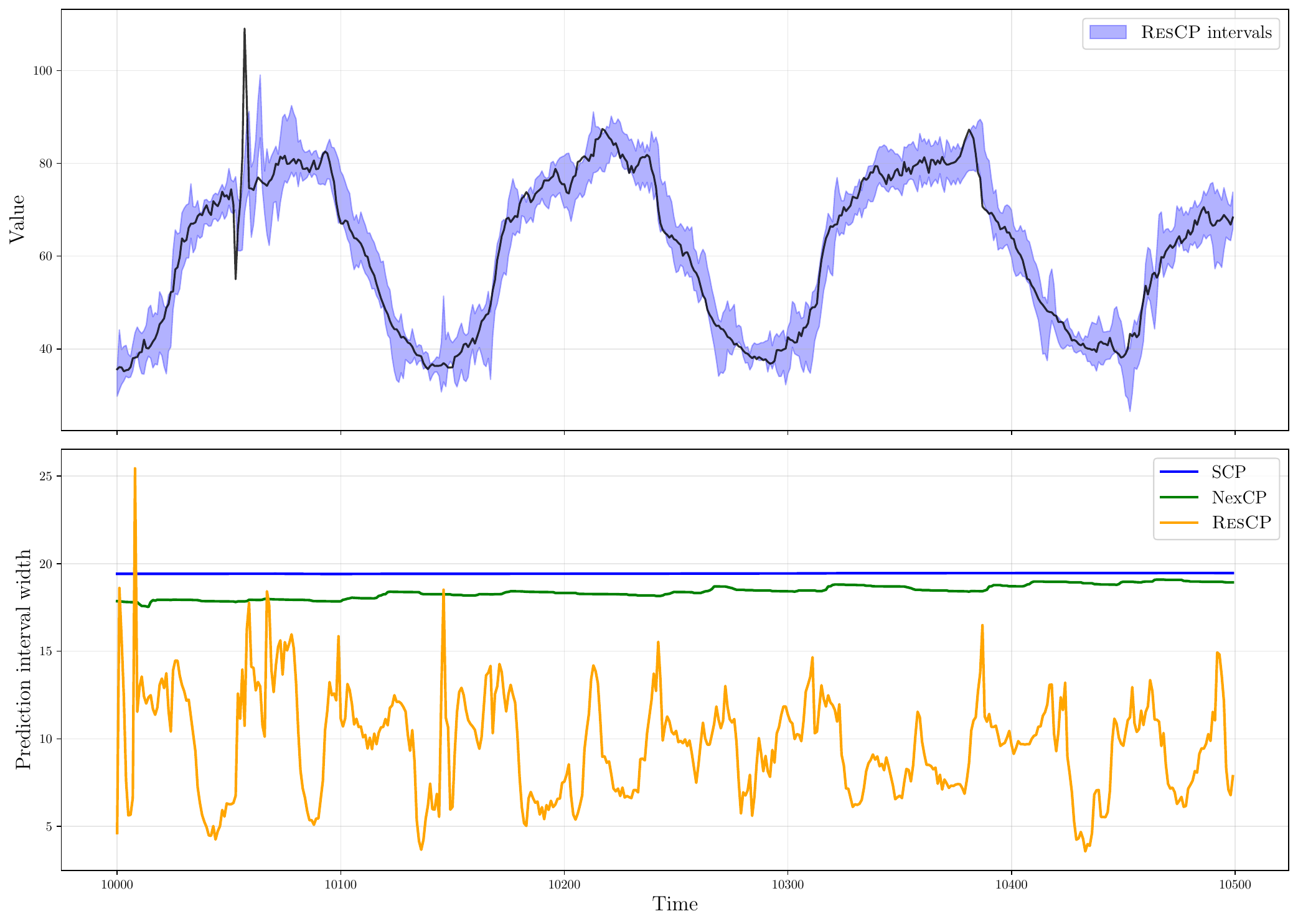}
    \caption{{Prediction intervals of \gls{rescp} (top) and prediction intervals' width of different methods (bottom) for the ACEA dataset}}
    \label{fig:elec_adaptive_pis}
\end{figure}

%% file: imgs/adaptive_pis_solar.tex
\begin{figure}
    \centering
    \includegraphics[width=0.9\linewidth, trim=50 20 50 0]{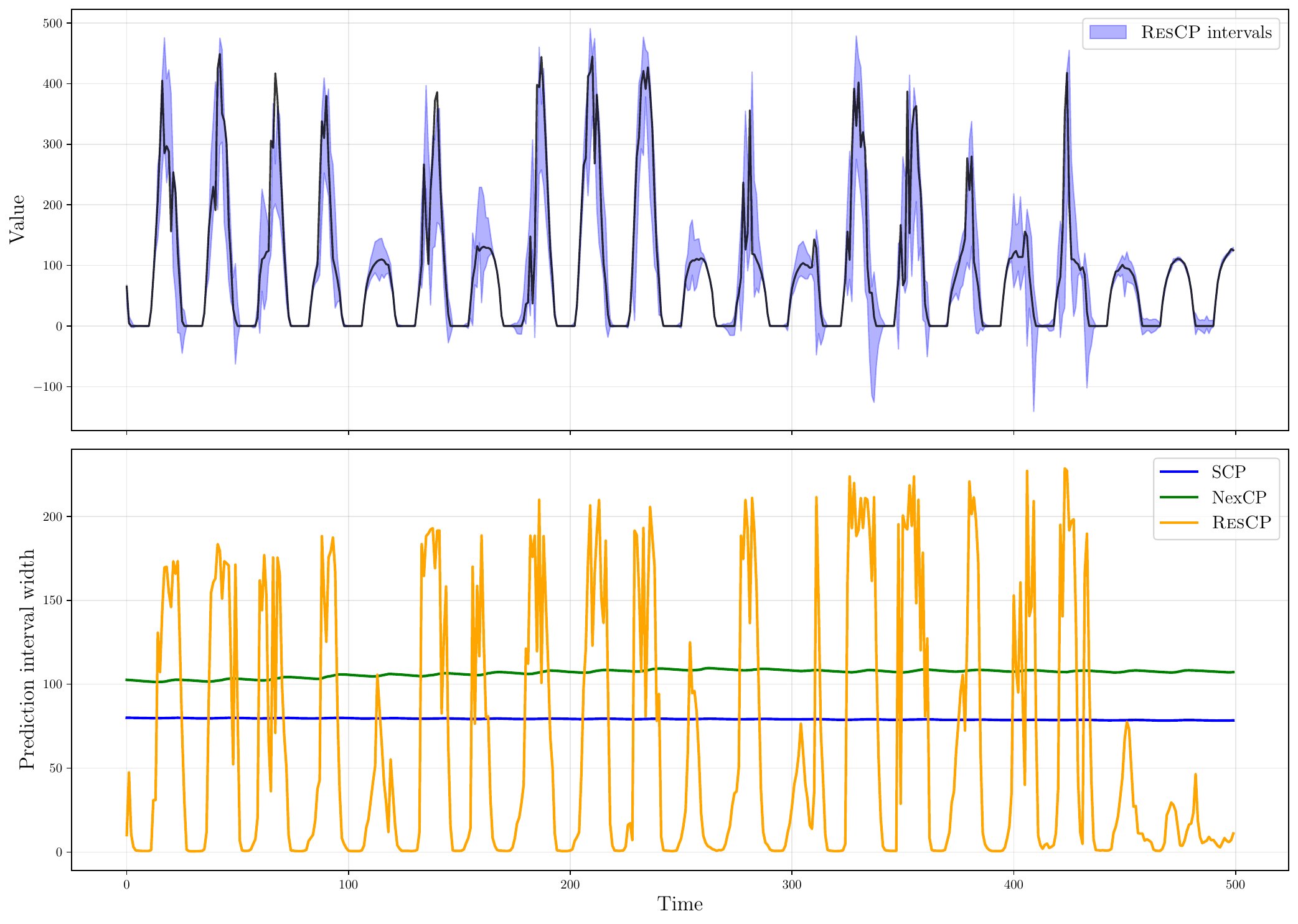}
    \caption{{Prediction intervals of \gls{rescp} (top) and prediction intervals' width of different methods (bottom) for the Solar dataset}}
    \label{fig:solar_adaptive_pis}
\end{figure}

%% file: appendix_ablation.tex
\section{{\gls{rescp} ablation study}}
\label{app:ablation}
\subsection{{Ablation on \gls{rescp} components}}
\input{ablation_table}
In this section, we report the remaining results of the ablation study on \gls{rescp} components presented in \autoref{sec:experiments} for the \gls{rnn} base model, extending them to all base models considered in our experiments. The results, summarized in \autoref{tab:ablation_study}, confirm the findings discussed in the main text.
Specifically, we see that both the time-dependent weights and the sliding window of past residuals significantly contribute to the performance of \gls{rescp} across all datasets and base models.
\subsection{{Use of exogenous variables}}
    \input{exog_comparison_table.tex}
    We further investigate the impact of incorporating exogenous variables into the \gls{rescp} framework for the Solar and Beijing datasets. To this end, we conduct experiments comparing the performance of \gls{rescp} with and without exogenous inputs using all base models. The results, summarized in \autoref{tab:exog_comparison}, show that the inclusion of exogenous variables in \gls{rescp} generally leads to worse performance in terms of coverage and Winkler score. This supports our choice of excluding exogenous variables from \gls{rescp}. Removing exogenous variables from \gls{rescqr}, instead, leads to clear degraded performances on the Solar dataset, while on the Beijing dataset the results are similar~(exogenous inputs do not appear particularly relevant to quantify uncertainty in this dataset). 

%% file: ablation_table.tex
\begin{table}[t]
\centering
\small
\caption{{Ablation of \textbf{\gls{rescp}}, \gls{rescp} without decay, \gls{rescp} without sliding window, and the combination of the two. Best Winkler scores are shown in bold.}}
\setlength{\tabcolsep}{3pt}
\setlength{\aboverulesep}{0pt}
\setlength{\belowrulesep}{0pt}
\renewcommand{\arraystretch}{1.05}

\begin{tabular}{@{}l|l|l|c|ccc@{}}
\multicolumn{2}{c}{} & Metric & \textbf{\gls{rescp}} & {\scriptsize No decay} & {\scriptsize No window} & {\scriptsize No window, no decay} \\
\midrule[1.5pt]
\multirow{9}{*}{\rotatebox{90}{Solar}}  & \multirow{3}{*}{\rotatebox{90}{RNN}} & $\Delta$Cov & 0.74{\tiny$\pm$0.24} & -0.10{\tiny$\pm$0.26} & 0.89{\tiny$\pm$0.20} & 1.70{\tiny$\pm$0.10} \\
 &                             & PI-Width & 62.25{\tiny$\pm$0.75} & 60.59{\tiny$\pm$0.84} & 59.81{\tiny$\pm$0.87} & 60.90{\tiny$\pm$0.30} \\
 &                             & Winkler & \textbf{104.24{\tiny$\pm$0.79}} & 107.46{\tiny$\pm$0.49} & 104.70{\tiny$\pm$0.20} & 104.46{\tiny$\pm$0.40} \\
\cmidrule{2-7}
 & \multirow{3}{*}{\rotatebox{90}{Transf}} & $\Delta$Cov & 3.09{\tiny$\pm$0.35} & 2.22{\tiny$\pm$0.37} & -0.37{\tiny$\pm$0.42} & 0.73{\tiny$\pm$0.54} \\
 &                             & PI-Width & 63.34{\tiny$\pm$1.11} & 58.49{\tiny$\pm$0.82} & 53.41{\tiny$\pm$0.74} & 56.37{\tiny$\pm$0.69} \\
 &                             & Winkler & 103.13{\tiny$\pm$0.58} & 103.35{\tiny$\pm$0.94} & 103.71{\tiny$\pm$0.39} & \textbf{101.36{\tiny$\pm$0.43}} \\
\cmidrule{2-7}
 & \multirow{3}{*}{\rotatebox{90}{ARIMA}} & $\Delta$Cov & 0.68{\tiny$\pm$0.95} & -1.06{\tiny$\pm$1.27} & 2.93{\tiny$\pm$0.34} & 2.97{\tiny$\pm$0.78} \\
 &                             & PI-Width & 77.17{\tiny$\pm$2.07} & 80.27{\tiny$\pm$2.90} & 82.68{\tiny$\pm$2.29} & 85.91{\tiny$\pm$2.82} \\
 &                             & Winkler & \textbf{110.38{\tiny$\pm$4.03}} & 113.09{\tiny$\pm$7.35} & 114.41{\tiny$\pm$1.98} & 118.20{\tiny$\pm$2.85} \\
\midrule
\multirow{9}{*}{\rotatebox{90}{Beijing}}  & \multirow{3}{*}{\rotatebox{90}{RNN}} & $\Delta$Cov & -0.70{\tiny$\pm$0.77} & -1.64{\tiny$\pm$0.61} & 0.10{\tiny$\pm$1.36} & 1.64{\tiny$\pm$0.88} \\
 &                             & PI-Width & 65.96{\tiny$\pm$2.50} & 64.04{\tiny$\pm$1.69} & 64.91{\tiny$\pm$3.36} & 69.53{\tiny$\pm$3.02} \\
 &                             & Winkler & 106.07{\tiny$\pm$0.47} & 114.85{\tiny$\pm$0.29} & \textbf{97.99{\tiny$\pm$0.31}} & 98.25{\tiny$\pm$0.66} \\
\cmidrule{2-7}
 & \multirow{3}{*}{\rotatebox{90}{Transf}} & $\Delta$Cov & -0.49{\tiny$\pm$0.59} & -1.53{\tiny$\pm$0.75} & -1.45{\tiny$\pm$0.73} & -1.09{\tiny$\pm$0.70} \\
 &                             & PI-Width & 64.06{\tiny$\pm$1.74} & 62.23{\tiny$\pm$2.17} & 62.37{\tiny$\pm$1.96} & 63.94{\tiny$\pm$2.01} \\
 &                             & Winkler & \textbf{103.64{\tiny$\pm$0.21}} & 112.31{\tiny$\pm$0.74} & 111.88{\tiny$\pm$0.50} & 112.40{\tiny$\pm$0.51} \\
\cmidrule{2-7}
 & \multirow{3}{*}{\rotatebox{90}{ARIMA}} & $\Delta$Cov & 0.63{\tiny$\pm$0.22} & 0.35{\tiny$\pm$0.23} & 1.31{\tiny$\pm$0.62} & 2.36{\tiny$\pm$0.41} \\
 &                             & PI-Width & 70.43{\tiny$\pm$0.86} & 70.90{\tiny$\pm$0.51} & 72.85{\tiny$\pm$2.17} & 77.26{\tiny$\pm$1.67} \\
 &                             & Winkler & \textbf{108.75{\tiny$\pm$0.31}} & 119.72{\tiny$\pm$0.44} & 118.14{\tiny$\pm$0.69} & 118.50{\tiny$\pm$0.46} \\
\midrule
\multirow{9}{*}{\rotatebox{90}{Exch.}}  & \multirow{3}{*}{\rotatebox{90}{RNN}} & $\Delta$Cov & 1.13{\tiny$\pm$0.27} & 2.01{\tiny$\pm$0.19} & 4.19{\tiny$\pm$0.05} & 4.20{\tiny$\pm$0.12} \\
 &                             & PI-Width & 0.0210{\tiny$\pm$0.0001} & 0.0219{\tiny$\pm$0.0001} & 0.0249{\tiny$\pm$0.0002} & 0.0254{\tiny$\pm$0.0003} \\
 &                             & Winkler & \textbf{0.0264{\tiny$\pm$0.0002}} & 0.0269{\tiny$\pm$0.0002} & 0.0284{\tiny$\pm$0.0001} & 0.0291{\tiny$\pm$0.0002} \\
\cmidrule{2-7}
 & \multirow{3}{*}{\rotatebox{90}{Transf}} & $\Delta$Cov & 1.46{\tiny$\pm$0.18} & 2.36{\tiny$\pm$0.07} & 4.45{\tiny$\pm$0.08} & 4.80{\tiny$\pm$0.14} \\
 &                             & PI-Width & 0.0229{\tiny$\pm$0.0001} & 0.0241{\tiny$\pm$0.0002} & 0.0265{\tiny$\pm$0.0003} & 0.0267{\tiny$\pm$0.0002} \\
 &                             & Winkler & \textbf{0.0294{\tiny$\pm$0.0001}} & 0.0298{\tiny$\pm$0.0002} & 0.0306{\tiny$\pm$0.0003} & 0.0307{\tiny$\pm$0.0003} \\
\cmidrule{2-7}
 & \multirow{3}{*}{\rotatebox{90}{ARIMA}} & $\Delta$Cov & 0.38{\tiny$\pm$0.41} & 1.61{\tiny$\pm$0.21} & 3.90{\tiny$\pm$0.21} & 3.40{\tiny$\pm$0.39} \\
 &                             & PI-Width & 0.0207{\tiny$\pm$0.0001} & 0.0215{\tiny$\pm$0.0001} & 0.0245{\tiny$\pm$0.0001} & 0.0248{\tiny$\pm$0.0002} \\
 &                             & Winkler & \textbf{0.0268{\tiny$\pm$0.0001}} & 0.0273{\tiny$\pm$0.0003} & 0.0288{\tiny$\pm$0.0002} & 0.0293{\tiny$\pm$0.0002} \\
\midrule
\multirow{9}{*}{\rotatebox{90}{ACEA}}  & \multirow{3}{*}{\rotatebox{90}{RNN}} & $\Delta$Cov & 1.56{\tiny$\pm$0.62} & 2.79{\tiny$\pm$0.45} & 5.34{\tiny$\pm$0.38} & 4.96{\tiny$\pm$0.63} \\
 &                             & PI-Width & 9.61{\tiny$\pm$0.26} & 10.15{\tiny$\pm$0.09} & 11.88{\tiny$\pm$0.09} & 12.15{\tiny$\pm$0.17} \\
 &                             & Winkler & \textbf{12.91{\tiny$\pm$0.23}} & 13.41{\tiny$\pm$0.08} & 14.80{\tiny$\pm$0.17} & 15.25{\tiny$\pm$0.40} \\
\cmidrule{2-7}
 & \multirow{3}{*}{\rotatebox{90}{Transf}} & $\Delta$Cov & 3.54{\tiny$\pm$0.32} & 3.29{\tiny$\pm$0.35} & 2.97{\tiny$\pm$0.47} & 3.04{\tiny$\pm$0.21} \\
 &                             & PI-Width & 10.10{\tiny$\pm$0.16} & 10.32{\tiny$\pm$0.05} & 9.34{\tiny$\pm$0.15} & 9.45{\tiny$\pm$0.07} \\
 &                             & Winkler & 12.90{\tiny$\pm$0.16} & 13.45{\tiny$\pm$0.15} & \textbf{12.27{\tiny$\pm$0.26}} & 12.49{\tiny$\pm$0.10} \\
\cmidrule{2-7}
 & \multirow{3}{*}{\rotatebox{90}{ARIMA}} & $\Delta$Cov & 5.02{\tiny$\pm$0.40} & 5.17{\tiny$\pm$0.33} & 7.68{\tiny$\pm$0.22} & 7.75{\tiny$\pm$0.24} \\
 &                             & PI-Width & 13.63{\tiny$\pm$0.55} & 13.50{\tiny$\pm$0.32} & 16.74{\tiny$\pm$0.72} & 16.91{\tiny$\pm$0.50} \\
 &                             & Winkler & \textbf{16.21{\tiny$\pm$0.53}} & 16.27{\tiny$\pm$0.45} & 19.24{\tiny$\pm$0.64} & 19.44{\tiny$\pm$0.45} \\
\bottomrule
\end{tabular}

\label{tab:ablation_study}
\end{table}

%% file: exog_comparison_table.tex
\begin{table}[t]
\centering
\small
\caption{{Comparison of \textbf{\gls{rescp}} and \textbf{\gls{rescqr}} with and without exogenous variables for Solar and Beijing datasets. Best Winkler scores (lowest) within each method are shown in bold.}}
\setlength{\tabcolsep}{3pt}
\setlength{\aboverulesep}{0pt}
\setlength{\belowrulesep}{0pt}
\renewcommand{\arraystretch}{1.05}

\begin{tabular}{@{}l|l|l|cc|cc@{}}
\multicolumn{2}{c}{} & Metric & \textbf{\gls{rescp}} & {\scriptsize w/ exog} & \textbf{\gls{rescqr}} & {\scriptsize w/o exog} \\
\midrule[1.5pt]
\multirow{9}{*}{\rotatebox{90}{Solar}} & \multirow{3}{*}{\rotatebox{90}{RNN}} & $\Delta$Cov & 0.74{\tiny$\pm$0.24} & 1.34{\tiny$\pm$0.09} & -1.10{\tiny$\pm$0.91} & -0.70{\tiny$\pm$0.12} \\
 &  & PI-Width & 62.25{\tiny$\pm$0.75} & 53.87{\tiny$\pm$0.59} & 59.99{\tiny$\pm$1.72} & 56.71{\tiny$\pm$0.17} \\
 &  & Winkler & \textbf{104.24{\tiny$\pm$0.79}} & 105.55{\tiny$\pm$0.81} & \textbf{82.76{\tiny$\pm$0.26}} & 88.46{\tiny$\pm$0.42} \\
\cmidrule{3-7}
 & \multirow{3}{*}{\rotatebox{90}{Transf}} & $\Delta$Cov & 3.09{\tiny$\pm$0.35} & 3.57{\tiny$\pm$0.02} & -3.51{\tiny$\pm$16.26} & 1.37{\tiny$\pm$1.08} \\
 &  & PI-Width & 63.34{\tiny$\pm$1.11} & 56.96{\tiny$\pm$0.27} & 59.56{\tiny$\pm$1.59} & 57.84{\tiny$\pm$0.36} \\
 &  & Winkler & 103.13{\tiny$\pm$0.58} & \textbf{101.73{\tiny$\pm$0.33}} & \textbf{82.16{\tiny$\pm$0.32}} & 88.30{\tiny$\pm$0.53} \\
\cmidrule{3-7}
 & \multirow{3}{*}{\rotatebox{90}{ARIMA}} & $\Delta$Cov & 0.68{\tiny$\pm$0.95} & -1.12{\tiny$\pm$0.13} & -2.03{\tiny$\pm$0.62} & -1.03{\tiny$\pm$0.59} \\
 &  & PI-Width & 77.17{\tiny$\pm$2.07} & 78.09{\tiny$\pm$0.47} & 66.19{\tiny$\pm$0.81} & 66.85{\tiny$\pm$0.93} \\
 &  & Winkler & \textbf{110.38{\tiny$\pm$4.03}} & 127.79{\tiny$\pm$1.00} & \textbf{85.38{\tiny$\pm$0.45}} & 88.19{\tiny$\pm$0.89} \\
\midrule
\multirow{9}{*}{\rotatebox{90}{Beijing}} & \multirow{3}{*}{\rotatebox{90}{RNN}} & $\Delta$Cov & -0.70{\tiny$\pm$0.77} & -4.71{\tiny$\pm$0.14} & -1.21{\tiny$\pm$1.65} & -0.53{\tiny$\pm$0.26} \\
 &  & PI-Width & 65.96{\tiny$\pm$2.50} & 56.29{\tiny$\pm$0.30} & 65.53{\tiny$\pm$4.09} & 65.11{\tiny$\pm$0.54} \\
 &  & Winkler & \textbf{106.07{\tiny$\pm$0.47}} & 115.73{\tiny$\pm$0.20} & 105.43{\tiny$\pm$0.85} & \textbf{104.93{\tiny$\pm$0.19}} \\
\cmidrule{3-7}
 & \multirow{3}{*}{\rotatebox{90}{Transf}} & $\Delta$Cov & -0.49{\tiny$\pm$0.59} & -1.87{\tiny$\pm$0.09} & -1.43{\tiny$\pm$1.10} & 0.20{\tiny$\pm$0.56} \\
 &  & PI-Width & 64.06{\tiny$\pm$1.74} & 59.52{\tiny$\pm$0.22} & 64.41{\tiny$\pm$2.72} & 66.49{\tiny$\pm$1.17} \\
 &  & Winkler & \textbf{103.64{\tiny$\pm$0.21}} & 112.90{\tiny$\pm$0.07} & 105.97{\tiny$\pm$1.21} & \textbf{105.44{\tiny$\pm$0.78}} \\
\cmidrule{3-7}
 & \multirow{3}{*}{\rotatebox{90}{ARIMA}} & $\Delta$Cov & 0.63{\tiny$\pm$0.22} & -0.96{\tiny$\pm$0.05} & -1.42{\tiny$\pm$1.15} & -0.16{\tiny$\pm$0.95} \\
 &  & PI-Width & 70.43{\tiny$\pm$0.86} & 67.51{\tiny$\pm$0.19} & 66.01{\tiny$\pm$3.01} & 67.33{\tiny$\pm$2.28} \\
 &  & Winkler & \textbf{108.75{\tiny$\pm$0.31}} & 120.59{\tiny$\pm$0.13} & 107.20{\tiny$\pm$1.21} & \textbf{103.97{\tiny$\pm$0.22}} \\
\bottomrule
\end{tabular}

\label{tab:exog_comparison}
\end{table}

%% file: appendix_sensitivity.tex
\section{Sensitivity analysis on ESN hyperparameters and temperature}\label{app:sensitivity}
\input{imgs/temperature_sensitivity}

After performing model selection and evaluating our model with the resulting hyperparameters, we conducted a small sensitivity analysis to determine how each part of the reservoir influenced model performance in \gls{rescp}, allowing us to assess the robustness of our chosen hyperparameters and identify the key aspects of the reservoir architecture that contribute most to quantification accuracy. We considered the \gls{rnn} point forecasting baseline.

\autoref{fig:temp_sens} shows how the value of the Winkler score and of the coverage varies as the temperature used in the \textsc{SoftMax} in equation~\autoref{eq:softmax} increases. As can be seen, there is a well-defined optimal range for the temperature in which the Winkler score is minimal and the coverage is close to the targeted one. Notice that for a small temperature, the method is unable to achieve the nominal coverage. This is the \textit{bias-variace tradeoff} that we mentioned in Section~\ref{sec:theory}, more specifically in Assumption~\ref{ass:weights}: the temperature should be small enough to localize only the meaningful residuals close to the query one, yet high enough to avoid collapsing the neighborhood to a single or limited set of similar samples.

\input{imgs/size_sensitivity}
\autoref{fig:sens_size} shows the behaviour of \gls{rescp} depending on the capacity of the reservoir. Apart from the Beijing dataset, the method can achieve the nominal coverage also when the dimensionality of the reservoir space is contained, at the price of a higher Winkler score. For the Beijing dataset, a bigger reservoir dimensionality is required to reach the target coverage.

\input{imgs/spectral_radius_sensitivity}
The sensitivity analysis of the spectral radius can be seen in \autoref{fig:sens_spectral}. It is interesting to see that in the Solar and Beijing datasets, the best performance is achieved even with $\rho(\mW_h) \geq 1$. This is due to the condition being only sufficient, not necessary: while $\rho(\mW_h ) < 1$ guarantees contractive dynamics in the autonomous linearized reservoir~\citep{gallicchio2020deep}, driven reservoirs with input scaling, leaky units, and bounded nonlinearities~\citep{dong2022asymptotic, ceni2025edge}.

\input{imgs/leak_sensitivity}
In \autoref{fig:sens_leak}, we observe how varying the leak rate affects the reservoir’s performance. In particular, a low value of the leaking rate slows the reservoir’s state updates, causing the network to retain past information for longer but reducing its responsiveness to new inputs. We can see how this heavily affects the observed coverage in almost all datasets.

\input{imgs/input_scaling_sensitivity}
Finally, the effect of different values of input scaling can be found in \autoref{fig:sens_input_scaling}. High values of input scaling lead the $\tanh$ activation function to saturation during state update, yielding loss of sensitivity to input variations. On the other hand, a small input scaling keeps the activation function in the linear part, limiting the richness of nonlinear dynamics necessary to capture complex temporal patterns.

%% file: imgs/temperature_sensitivity.tex
\begin{figure}[h]
    \centering
    \includegraphics[width=0.9\linewidth, trim=50 20 50 10]{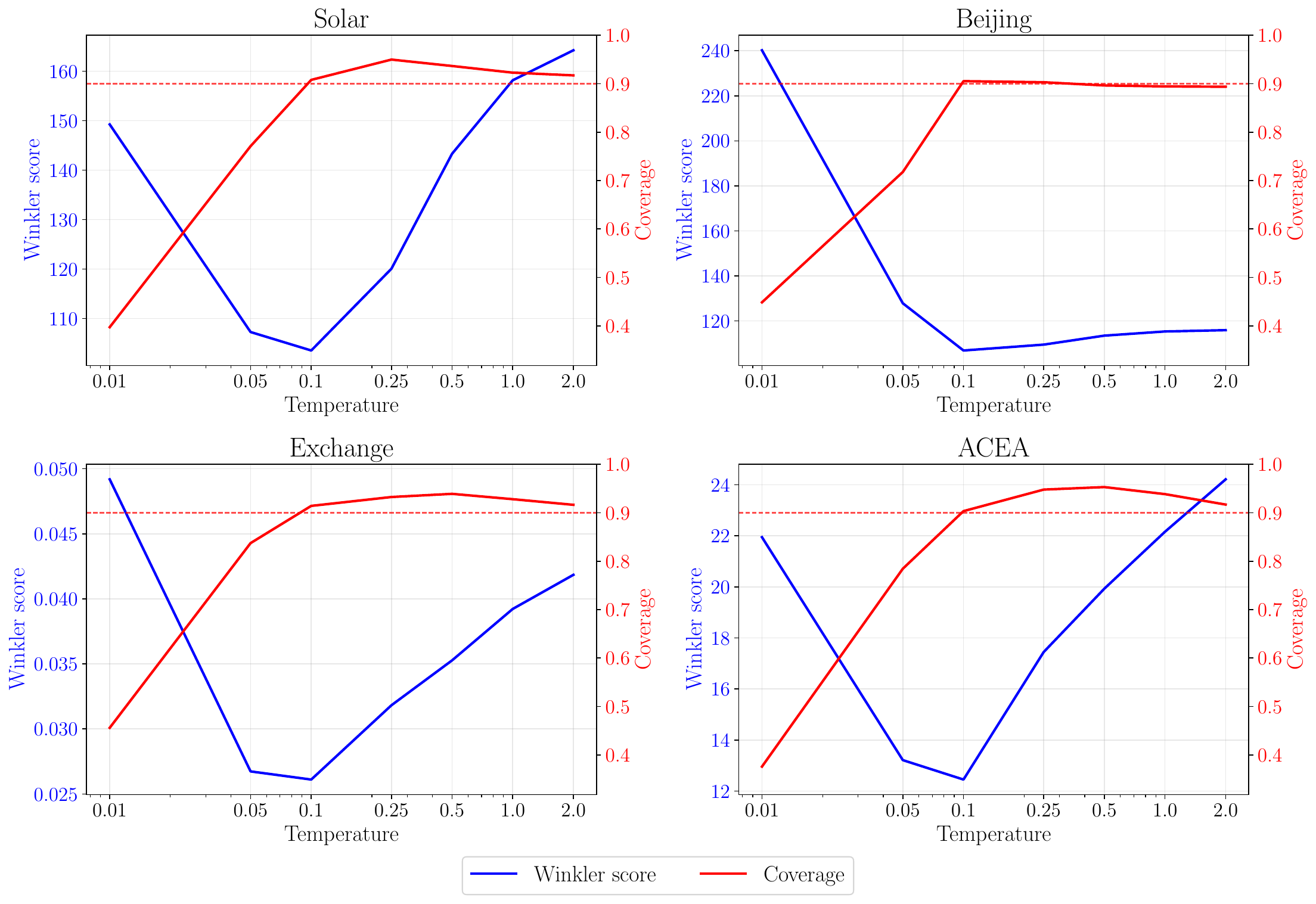}
    \caption{Sensitivity analysis of the temperature.}
    \label{fig:temp_sens}
\end{figure}


%% file: imgs/size_sensitivity.tex
\begin{figure}
    \centering
    \includegraphics[width=0.9\linewidth, trim=50 20 50 10]{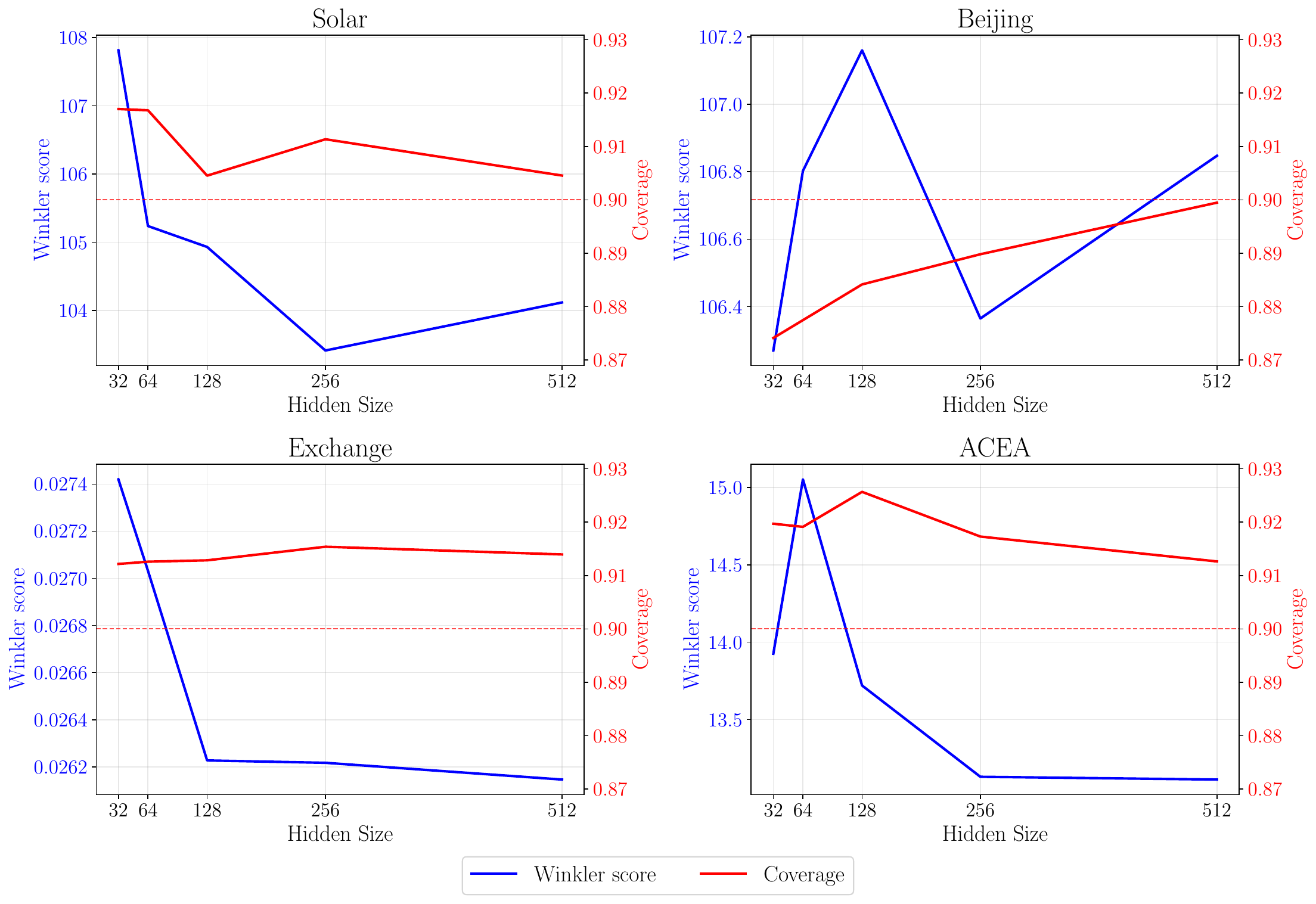}
    \caption{Sensitivity analysis of the reservoir size.}
    \label{fig:sens_size}
\end{figure}

%% file: imgs/spectral_radius_sensitivity.tex
\begin{figure}
    \centering
    \includegraphics[width=0.9\linewidth, trim=50 20 50 10]{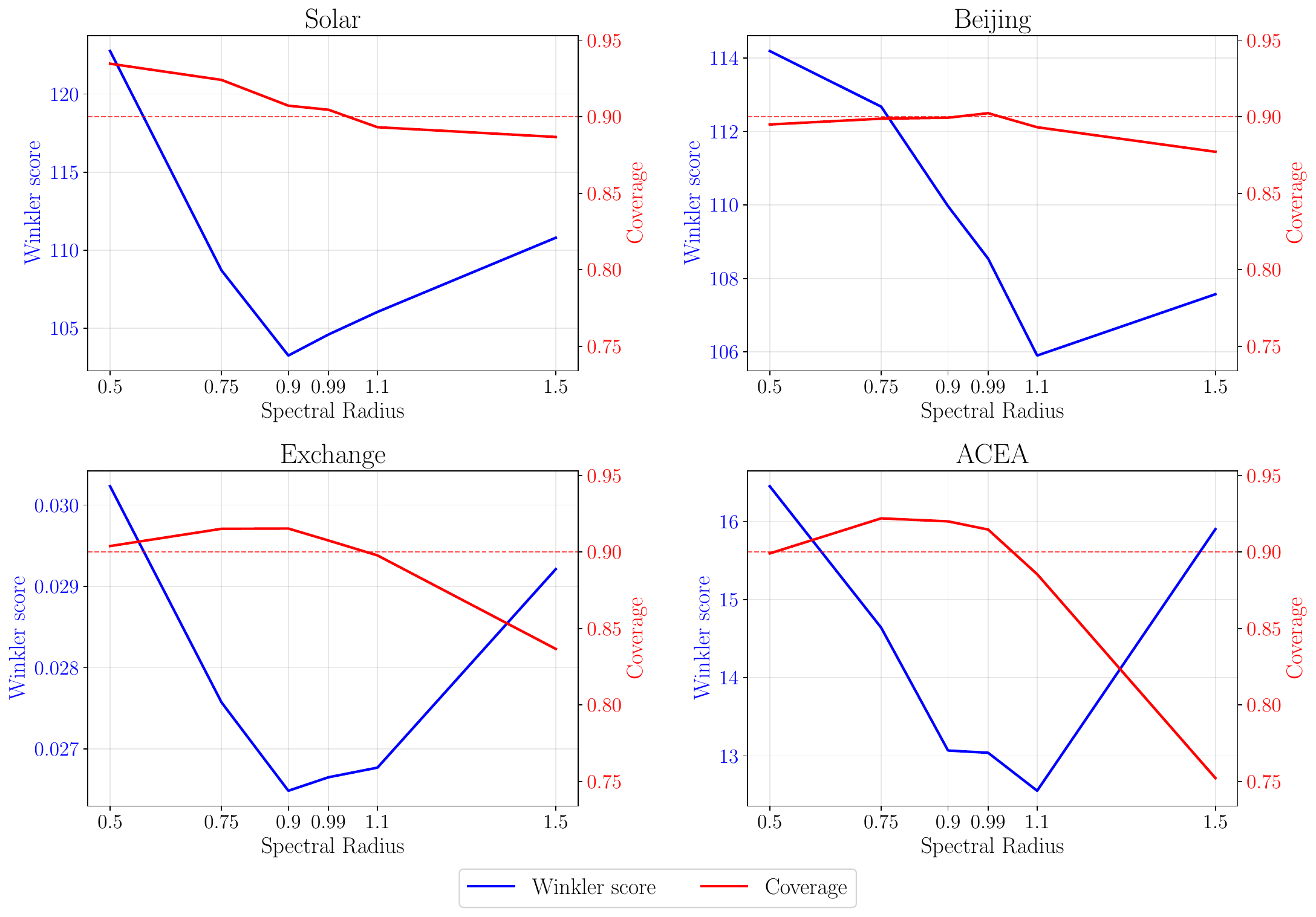}
    \caption{Sensitivity analysis of the spectral radius.}
    \label{fig:sens_spectral}
\end{figure}

%% file: imgs/leak_sensitivity.tex
\begin{figure}
    \centering
    \includegraphics[width=0.9\linewidth, trim=50 20 50 10]{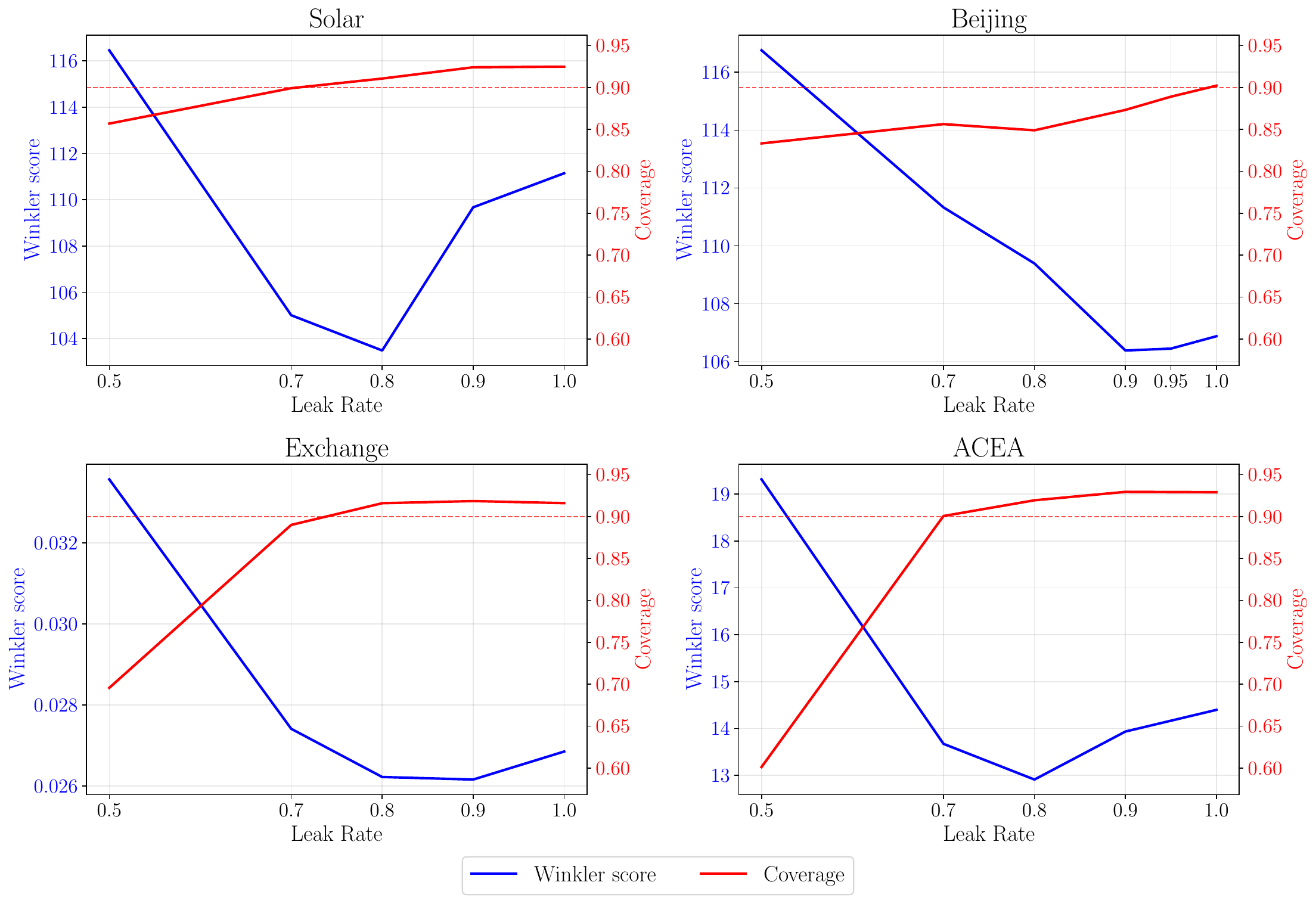}
    \caption{Sensitivity analysis of the leak rate.}
    \label{fig:sens_leak}
\end{figure}

%% file: imgs/input_scaling_sensitivity.tex
\begin{figure}
    \centering
    \includegraphics[width=0.9\linewidth, trim=50 20 50 10]{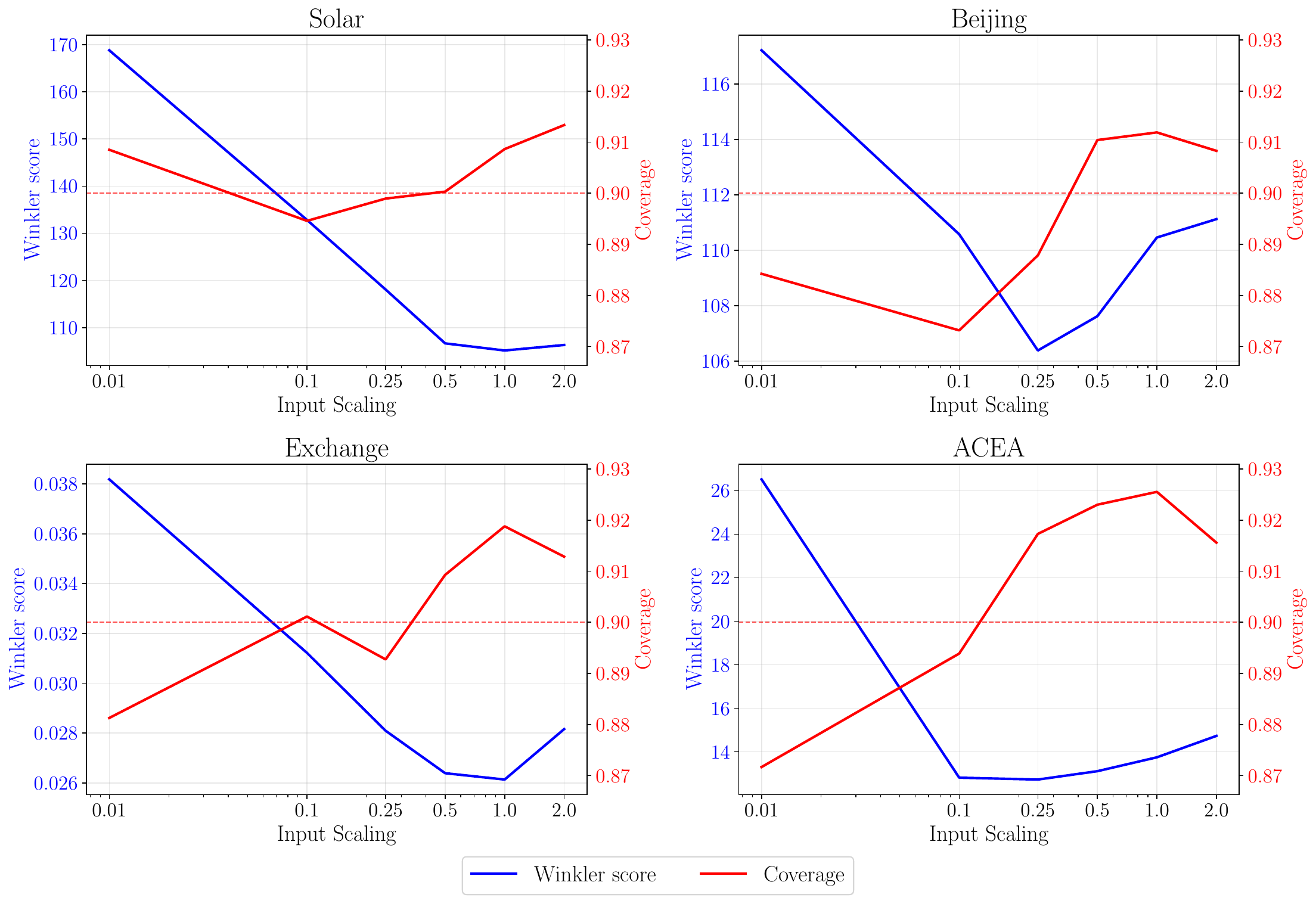}
    \caption{Sensitivity analysis of the input scaling.}
    \label{fig:sens_input_scaling}
\end{figure}

%% file: appendix_hw_sw.tex
\section{{Implementation Details}}\label{app:reproducibility}
{
The code have been written in Python~\citep{van2009python} using the following open-source packages:
\begin{itemize}
    \item Numpy~\citep{harris2020array};
    \item Pandas~\citep{reback2020pandas,mckinney2010data};
    \item PyTorch~\citep{paszke2019pytorch};
    \item PyTorch Lightning~\citep{falcon2019pytorch};
    \item Torch Spatiotemporal~\citep{cini2022torch};
    \item \texttt{reservoir-computing}~\citep{bianchi2020reservoir}.
\end{itemize}
Experiments were conducted on a machine equipped with a AMD Ryzen 9 7900X CPU and a NVIDIA GeForce RTX 4090.
}

%% file: llms.tex
\section{Use of Large Language Models} 
We acknowledge the use of Large Language Models as a writing tool for minor edits in single sentences.

\section{Ethics Statement}
The work presented in this paper is about basic machine learning research and we perform experiments on standard, publicly available datasets. 
The authors have read and adhere to the ICLR Code of Ethics and do not foresee any direct ethical concerns or potential for misuse.

\section{Reproducibility Statement}
The code to reproduce the experiments presented in this paper will be made publicly available.